\documentclass[12pt]{article}

\usepackage{graphics}
\usepackage{epsfig}

\usepackage{amssymb}
\usepackage{amsthm}

\usepackage{color, palatino}
\usepackage[top=3cm, bottom=3cm, left=2cm, right=2cm]{geometry}
\usepackage{epsfig}
\usepackage{graphics}
\usepackage{enumerate}
\usepackage{amstext}
\usepackage{amsmath}
\usepackage{pdfsync}

\usepackage{bbm}

\newtheorem{thm}{Theorem}
\newtheorem{lem}{Lemma} 
\newtheorem{prop}{Proposition}

\newtheorem{defn}{Definition}
\newtheorem{assm}{Assumption}

\newtheorem{rem}{Remark}

\newcommand{\new}{\newcommand}

\providecommand{\nor}[1]{\left\lVert {#1} \right\rVert}
\providecommand{\abs}[1]{\left\lvert{#1}\right\rvert}
\providecommand{\set}[1]{\left\{#1\right\}}
\providecommand{\scal}[2]{\left\langle{#1},{#2}\right\rangle}
\newcommand{\col}[1]{\left(\begin{array}{c} #1 \end{array}\right)}

\usepackage{comment}

\new{\R}{\mathbb R}
\new{\C}{\mathbb C}
\new{\N}{\mathbb N}
\new{\Z}{\mathbb Z}
\new{\hh}{\mathcal H}
\new{\vv}{\mathcal V}
\new{\la}{\lambda}
\new{\eps}{\epsilon}
\new{\pphi}{\varphi}
\new{\ran}[1]{\operatorname{ran}\, #1}
\new{\tr}[1]{\operatorname{tr}\left[ #1 \right]}
\new{\supp}[1]{\operatorname{supp}#1}
\new{\inte}[1]{\operatorname{int}#1}
\new{\frecc}{\longrightarrow}
\new{\PP}[1]{{\mathbb P}\left(#1\right)}
\new{\EE}[1]{{\mathbb E}[#1]}
\new{\A}[1]{\mathcal{A}_{#1}}
\new{\id}{\mathbbm{1}}
\new{\dk}{d_K}
\new{\G}{g_\la}
\new{\Gn}{g_{\la_n}}
\new{\RR}{r_\la}
\new{\RRn}{r_{\la_n}}
\new{\ff}{\mathcal{F}}
\new{\ind}{1{\hskip -2.5 pt}\hbox{I}}
\newcommand{\yps}{y}

\begin{document}

\title{Learning Sets  with Separating Kernels}

\author{Ernesto De Vito\thanks{DIMA,  Universit\`a di Genova,  Genova, Italy. E-mail: {\em devito@dima.unige.it}} \and Lorenzo Rosasco\thanks{DIBRIS, Universit\'a di Genova \& Istituto Italiano di Tecnologia, Italy, \&~Massachusetts Institute of Technology, U.S.A.. E-mail: {\em lrosasco@mit.edu}} \and Alessandro Toigo\thanks{Dipartimento di Matematica, Politecnico di Milano, Milano, Italy \& I.N.F.N., Sezione di Milano, Milano, Italy. E-mail: {\em alessandro.toigo@polimi.it}}}

\date{\empty}

\maketitle

\begin{abstract}
  We consider the problem of learning a set from random samples.  We
  show how relevant geometric and topological properties of a set can
  be studied analytically using concepts from the theory of
  reproducing kernel Hilbert spaces.  A new kind of reproducing
  kernel, that we call separating kernel, plays a crucial role in
  our study and is analyzed in detail.  We prove a new analytic
  characterization of the support of a distribution, {that naturally leads to a family of
   regularized learning algorithms which are 
   provably universally consistent and 
     stable with respect to random sampling.}
  Numerical experiments show that the proposed approach is competitive, and
  often better, than other state of the art techniques.\\
\end{abstract}

\tableofcontents

\section{Introduction}

In this paper we study the problem of learning from data the set
  where the data probability distribution is concentrated.  Our study
is more broadly motivated by questions in unsupervised learning, such
as the problem of inferring geometric properties of probability
distributions from random samples.
 
In recent years,   there has been great progress in the  
theory and  algorithms for supervised learning, i.e.~function approximation 
problems from random noisy data \cite{bouboulug04b,cucsma02,degylu96,PogSma03,vapnik98}.  
On the other hand,  while there are a number of  methods and studies in unsupervised learning, 
e.g.  algorithms for  clustering, dimensionality reduction,  dictionary learning
 (see Chapter~14 of~\cite{hatifr01}),  many interesting problems  remain largely unexplored.

 Our analysis  starts with the observation that many studies in unsupervised learning hinge 
on at least one of the following  two assumptions. The first is that the data are distributed 
according  to a probability distribution which is absolutely continuous with respect to a reference 
measure, such  as~the Lebesgue measure. In this case it is possible to define a density and 
the corresponding  density level sets. Studies in this scenario include \cite{bibema09,dewi80,kots93,sthusc05} to name a few. 
Such an assumption prevents considering the case where the data are 
represented in  a high dimensional Euclidean space  but  are
concentrated on a Lebesgue negligible subset, as a lower dimensional submanifold. This motivates the second assumption -- sometimes called 
 {\em manifold assumption} -- postulating that the data lie on  a low dimensional
Riemannian manifold embedded in an Euclidean space.  This latter idea 
has triggered a large number  of different algorithmic and theoretical studies (see  for example \cite{beni03,beni08,yale1,yale2,isomap,RSLLE}).
Though the manifold assumption   has proved useful in some  applications, there are many practical scenarios where it might not be satisfied.
This observation  has motivated considering more general situations such 
as {\em manifold plus noise} models \cite{limaro11,NiSmWe08}, and models where the data are 
described by combinations of more than one manifold \cite{lete10,vimasa05}.

Here we consider a different point of view and work in a  setting where the data 
are described by an abstract probability space and a {\em similarity function} induced by  a reproducing kernel  \cite{smazho08}. 
In this  framework, we consider the basic problem of  estimating the set where 
the data distribution is  concentrated (see  Section \ref{PW}  for a detailed discussion of related works).
A special  class of  reproducing kernels, that we call separating kernels, plays a special 
role in our study. First, it allows to define a suitable metric on the probability space and  makes the support of the distribution well defined; second, it leads to a new analytical characterization 
of the support in terms of the null space of the integral operator {associated to} the reproducing kernel. 

This last result is the key towards a new computational approach to
learn the support from data, since the integral operator can be
approximated with high probability from random samples
\cite{RoBeDe10,smazho08}.  {Estimation of  the null space of the integral
  operator can be unstable, and
  regularization techniques can be used to obtain stable
  estimators.} In this paper we study a class of regularization
  techniques proposed to solve ill-posed problems \cite{enhane} and
  already studied in the context of supervised learning
  \cite{bapero07,logerfo08}.  Regularization is achieved by {\em
    filtering} out the small eigenvalues of the sample empirical
  matrix defined by the kernel.  Different algorithms are defined by
  different filter functions and have different computational
  properties.
Consistency and stability properties for a large class of spectral
filters {and of the corresponding  algorithms are established}  in a
unified framework. Numerical experiments show that the proposed
algorithms  are competitive, and often better, than  other state of
the art techniques.  

The paper is divided into two parts.  The first part includes  Section \ref{sec:basic}, where we establish several 
mathematical results relating reproducing  kernel Hilbert spaces {of functions on a set $X$ and the 
geometry of the set $X$ itself.} In particular, in this section we introduce the concept of  separating 
kernel,  which we further explore in Section \ref{RKHSCR}. 
These results are of interest in their own right, and are at the heart of our approach.
In the second {part of the} paper we discuss the problem of learning the  support from data.
{More precisely, in} Section \ref{sec:learn} we {illustrate some} algorithms for learning the  
support of a distribution  from random {samples. In Section \ref{consistency} we establish universal consistency for the proposed methods and discuss   stability to random sampling.}
We conclude in Section  \ref{sec:disc} and \ref{sec:emp} with some further discussions and some numerical experiments, respectively.
A conference  version of this paper appeared in \cite{deroto10}.
We now start by describing in some more detail our results and discussing some related works.

\subsection{Summary of main results}\label{sec:summary}

In this section we briefly describe the main ideas and results in the paper.

The setting we consider is described by a probability space $(X,\rho)$  and a measurable reproducing kernel $K$ on the set $X$ \cite{aron50}. The data are independent and identically distributed (i.i.d.) samples $x_1, \dots, x_n$, each one drawn from $X$ with probability $\rho$. The reproducing kernel $K$ reflects some prior information on the problem and,  as we discuss in the following,  will also define the geometry of $X$. The goal is to use the sample points $x_1, \dots, x_n$ to estimate the region where the probability measure $\rho$ is concentrated.

To fix some ideas,   the space $X$ can  be thought of as a high-dimensional Euclidean space and  the distribution $\rho$ as  being  concentrated on a region $X_\rho$, which is a   smaller {-- and   potentially lower dimensional --
subset of $X$ (e.g.~a linear} subspace or a manifold).  In this example, the goal is to build from data an estimator $X_n$ which is, with high probability, close to $X_\rho$  with respect to a suitable metric.

{We first note that a precise definition of $X_\rho$  requires some care. If $\rho$ is assumed to have  
a {continuous} density with respect to some fixed reference measure (for example, the Lebesgue measure in the Euclidean space){, then}  the region $X_\rho$ can be easily defined to be the closure of the set of points where the density function is non-zero. {Nevertheless, this assumption} would prevent considering
the situation where the data are concentrated on a ``small'', possibly lower dimensional, subset of $X$.  {Note that, if} the set $X$ {were} endowed with a topological structure and $\rho$ {were} defined on the corresponding Borel $\sigma$-algebra, it {would be} natural to define $X_\rho$ as the support of the measure $\rho$, i.e.~the smallest {\em closed} subset of $X$ having measure one.  However, since the set $X$ is only assumed to be a measurable space, no a priori given topology is available. {Here we also remark} that the definition of $X_\rho$ is not the only point where some further structure on $X$ would be useful.  Indeed, when defining a learning error, a notion of distance between the set $X_\rho$ and its estimator $X_n$ is also needed and hence some metric structure on $X$ is required.}

{{The idea is to use} the properties of  the reproducing kernel $K$  to induce a metric structure {-- and consequently a topology --} on $X$.
Indeed, under some mild technical assumptions on $K$, the function
\[
\dk(x,\yps)= \sqrt{K(x,x)+K(\yps,\yps)-{2}K(x,\yps)} \qquad \forall ~{x,y\in X}
\]
defines a metric on $X$, thus making  $X$ a topological space. {Then, it is natural} to define $X_\rho$ to be the support of $\rho$ with respect to such metric topology. {Moreover, the 
Hausdorff distance $d_H$ induced  by the metric $\dk$  provides  {a notion} of distance between closed sets.}

{The problem we consider  can now be restated as  follows:} we want to learn from data an {estimator} $X_n$ of $X_\rho$, such that {$\lim_{n\to\infty} d_H(X_n,X_\rho)=0$ almost surely.}
While $X_\rho$ is now well defined, it is not clear how to build an estimator from data. 
A main result in the paper, given in Theorem \ref{primo}, 
provides a new analytic characterization of $X_\rho$, which immediately suggests
a new computational solution for the corresponding learning problem. 
To derive and state this result, we introduce a new notion of  reproducing  kernels, called separating kernels,  that, roughly speaking,
 {captures} the sense in which  the reproducing kernel  and the probability distribution need to be related.
We say that  a reproducing kernel Hilbert space {$\hh$} (or equivalently its kernel) {\em separates}   a subset $C\subset X$,  if, for any
$x\not\in C$, there exists $f\in\hh$ such that
$$
f(x) \neq 0 \quad\text{ and }\quad f(\yps)= 0 \quad \forall \yps\in C .
$$
If $K$ separates all possible closed subsets in $X$, we say that it is {\em completely separating}.
Figure~\ref{separability} illustrates the notion of separating kernel in the simple example of the linear kernel in a Euclidean space.

\begin{figure}[t!]
\begin{center}
\includegraphics[width=5in]{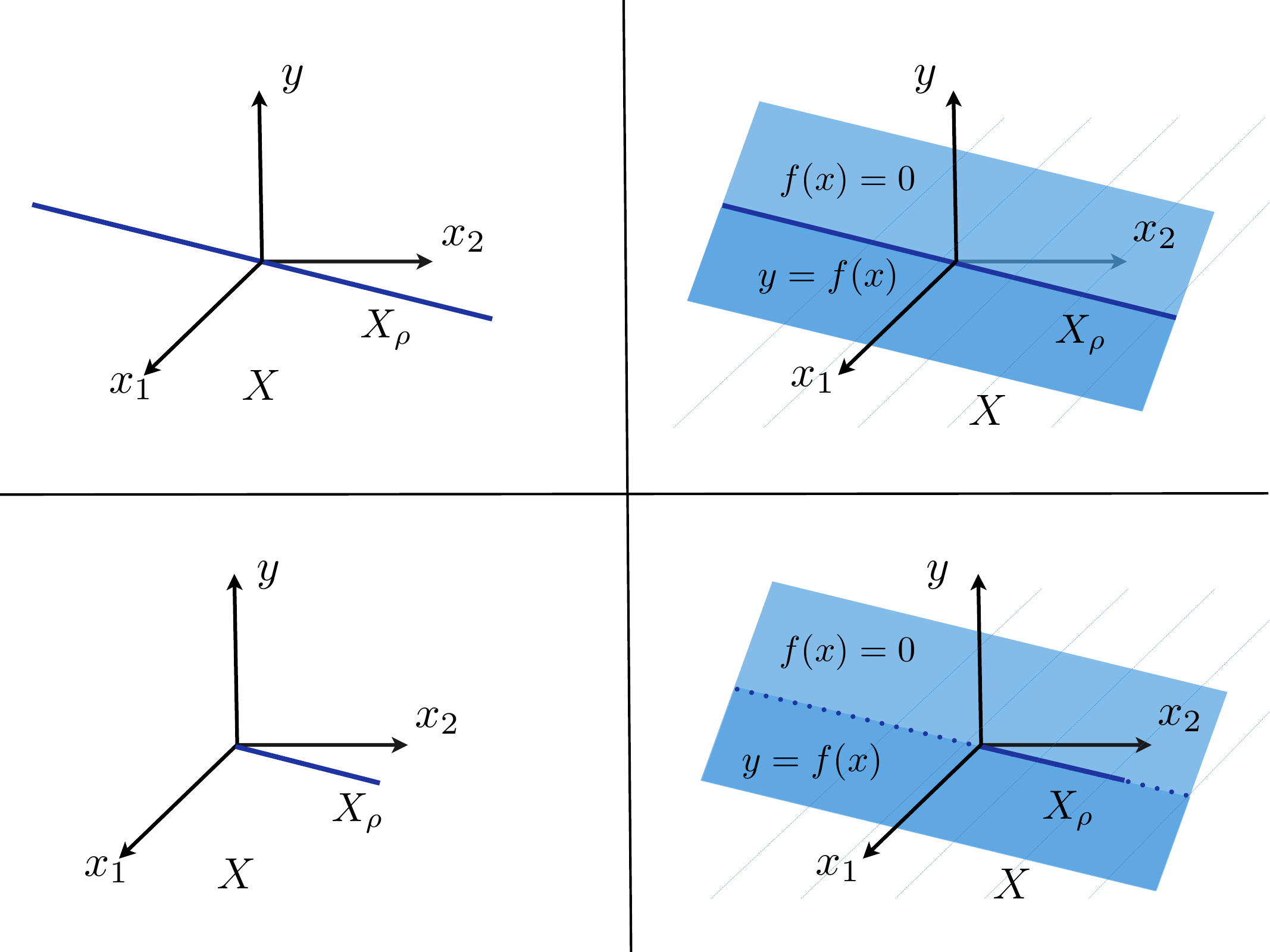}

\end{center}

\caption{The separating property is illustrated in a simple situation where $X={\R^2}$. In the top pictures, the support $X_\rho$ is a line passing through the origin and is separated by the linear kernel $K(x,\yps)=x^T \yps$: for all $x\notin X_\rho$, there exists a function $f\in\hh$ (a linear function on $X$) which is zero on $X_\rho$ and such that $f(x) \neq 0$. The pictures on the right are a plot of the plane $y={f(x_1,x_2)}$. In the bottom pictures, the support is a segment passing through the origin. The linear kernel is too simple to separate this set: all planes are going to be zero also outside of the support (the dotted line in the picture).}\label{separability}
\end{figure}

Now, Theorem \ref{primo} states that, if either $K$ is completely separating, or at least separates $X_\rho$, then $X_\rho$ is the level set of a suitable distribution dependent continuous function $F_\rho$. More precisely, let  ${\hh}$ be the reproducing kernel Hilbert space associated to $K$ \cite{aron50}, $T:\hh\to\hh$  the integral operator with kernel $K$, and denote by $T^\dag$ its pseudo-inverse. If we consider the function $F_\rho$ on $X$, defined by
\begin{equation*}\label{mainintro2}
F_\rho(x)=\scal{T^\dagger T K_x}{K_x} \qquad \forall x\in X,
\end{equation*}
and  $K$ separates $X_\rho$, then we prove that
\begin{equation*}
X_\rho=\set{x\in X\mid F_\rho(x)=1},
\end{equation*}
(where for simplicity we are assuming  $K(x,x)=1$ for all $x\in X$).

The above result is crucial since the integral operator $T$ can be approximated with high probability from 
data (see \cite{RoBeDe10} and references therein). However, since the definition of $F_\rho$ involves the pseudo-inverse of $T$, the support estimation problem can be unstable  \cite{tikars77} and regularization techniques are needed to ensure stability. 
With this in mind, we propose and study a family  of spectral regularization techniques which are classical in inverse problems \cite{enhane} and have been considered in supervised learning in  \cite{bapero07,logerfo08}. We define an estimator  by  
\[X_n=\{x\in X~|~ F_n(x)\geq 1-\tau_n\},\]
where $F_n(x)= (1/n) {\mathbf K}_{{\mathbf x}}^{\top}\,
g_{\la_n}({\mathbf K_n}/n){\mathbf K}_{{\mathbf x}}$,  with ${\mathbf
  ({\mathbf K}_n)}_{i,j}=K(x_i,x_j)$, ${\mathbf K}_{{\mathbf x}}$ is
the column vector whose $i$-th entry is $K(x_i,x)$, and ${\mathbf
  K}_{{\mathbf x}}^{\top}$ is its {transpose}. Here $g_{\la_n}({\mathbf K_n}/n)$ is a matrix defined via spectral calculus
by a spectral filter function $g_{\la_n}$ that  suppresses the contribution of the  eigenvalues smaller than $\la_n$. {Examples of spectral filters include Tikhonov regularization and truncated singular values decomposition \cite{logerfo08}, to name  a few.}

{This class of methods can be studied within a unified framework, and  the  
error analysis in the paper establishes strong universal consistency  if $X_\rho$ is separated by $K$. More precisely, under the latter assumption, we  show in Theorem~\ref{thm:hauss} that,  
\[ \lim_{n\to \infty} d_H (X_n, X_\rho)= 0 \qquad \text{almost surely},\]
provided that $X$ is compact and  the sequences $(\tau_n)_{n\geq 1},(\la_n)_{n\geq 1}$ are chosen so that, 
$$
{\tau_n= 1-\min_{1\leq i \leq n}F_n(x_i),} \quad\quad \lim_{n\to\infty}  \la_n=0,\quad\quad \sup_{n\geq 1} (L_{\la_n} \log n)/\sqrt{n}<+\infty,
$$ where $L_{\la_n}$ is the Lipschitz constant of the function $r_{\la_n}(\sigma)=\sigma g_{\la_n}(\sigma)$.  The above result is universal in the sense that consistency can be shown 
without assuming regularity condition on $\rho$ or $X_\rho$.

The proof of the above result crucially depends on estimating the deviation 
between $F_n$ and  $F_\rho$. Indeed, for the above choice of the sequence 
$\la_n)_{n\geq 1}$ we show that  
\[
\lim_{n\to \infty} \sup_{x\in X}\abs{F_\rho(x)-F_n(x)}=0 \qquad \text{almost surely}.
\]
Under suitable distribution dependent assumptions,  the above result can be further developed to obtain finite sample bounds   quantifying stability to random sampling. 
Indeed, if the couple $(\rho, K)$ is such that {$\sup_{x\in X}\nor{T^{-s/2} T^\dagger T K_x}<+\infty$}, with $0< s\leq 1$, and the eigenvalues of the 
(compact and positive) operator $T$ satisfy  $\sigma_j\sim
j^{-1/b}$ for some $0< b\leq 1$,  
then  we prove in Theorem \ref{rates} that, for $n\geq 1$ and $\delta > 0$, we have
$$ \sup_{x\in X}  \abs{F_n(x)-F_\rho(x)} \leq C_{s,b,\delta} \left(\frac{1}{n}\right)^{\frac{s}{2s+b+1}}
$$
with probability at least $1-2e^{-\delta}$, for  $\la_n =n^{-1/(2s+b+1)}$ and a suitable constant $C_{s,b,\delta}$ which does not depend on $n$.
}

Finally,  we remark that our construction relies on the assumption that  the kernel $K$ separates the support $X_\rho$. 
The question then arises whether there exist kernels that can separate a large number of, and perhaps {\em all}, closed subsets, namely kernels that are \emph{completely separating}. {Indeed, 
a positive answer can be given and, for} translation invariant kernels on $\R^d$, Theorem~\ref{ale} actually gives a  sufficient condition for a kernel to be completely separating in terms of its Fourier transform. As a consequence,  the Abel kernel {$K(x,\yps)=e^{-\nor{x-\yps}/\sigma}$} on the Euclidean space $X=\R^d$ is completely separating. Interestingly, the  Gaussian kernel {$K(x,\yps)=e^{-\nor{x-\yps}^2 /\sigma^2}$}, which is very popular  in machine learning,  is not.

\subsection{State of the art}\label{PW}

The problem of building an estimator $X_n$ of  a subset $X_\rho\subset X$ which is consistent with respect to some kind of
metric among sets has been considered in seemingly diverse fields for different application purposes, 
from anomaly detection -- see   \cite{chbaku09} for a review --  to surface estimation \cite{scgisp05}.
We give a  summary of the main approaches, with basic  references for further details. \\
{\bf Support and Level Set Estimation}.
Support estimation (also called set estimation) is a part of the theory of non-parametric  statistics. We refer to \cite{cufr10,curo03} for a detailed review on this topic.
Usually, the space $X$ is $\R^d$ with the Euclidean metric $d$, and
$X_\rho$ is the corresponding support of $\rho$.  If $X_\rho$ is convex, a
natural estimator is the convex hull of the data
$X_n=\operatorname{conv}\set{x_1,\ldots,x_n}$, for which convergence rates
can be derived with respect to the Hausdorff
distance \cite{duwa96,rei03}.
If $X_\rho$ is not convex,  Devroye and Wise \cite{dewi80} propose the estimator
\[ X_n=\bigcup_{i=1}^n B(x_i,\eps_n),\]
where $B(x,\eps)$ is the ball of center $x$ and radius $\eps$, and $\eps_n$  slowly goes to zero when $n$ tends to
infinity. Consistency and minimax  converges rates are studied in
\cite{dewi80,kots93} with respect to the distance
\[ d_{\mu}(C_1,C_2)=\mu(C_1\triangle C_2),\]
where $C_1\triangle C_2=(C_1\setminus C_2)\cup (C_2\setminus C_1)$
and $\mu$ is a suitable known measure.\\
If $\rho$ has a density $f$ with respect to some known measure $\mu$, a traditional approach is based on   a non-parametric estimator $f_n$ of $f$, a so called {\em plug-in} estimator.
A kernel based class of plug-in estimators is proposed  in \cite{cufr97}, namely
\[ {X_n=\set{x \in X\mid f_n(x) \geq c_n} \quad \text{ with } \quad f_n(x)=\frac{1}{n h_n^d}\sum_{i=1}^n K\left(\frac{x-x_i}{h_n}\right),}\]
where $h_n$ is a regularization parameter and $c_n$ is a suitable
threshold.  Convergence rates with respect to $d_\mu$ are provided in \cite{cufr97}. \\
A related problem   is  level set estimation, where the goal is to detect the high density
regions $\set{x\in X\mid f(x)\geq c}$.  Consistency and optimal convergence rates for different  plug-in
estimators
$$
X_n = \set{x\in X\mid f_n(x)\geq c}
$$
have been studied  with respect to both 
$d_H$ and $d_\mu$, see for example \cite{bibema09,scno06,tsy97}  for a
slightly different approach.\\
{\bf One class learning algorithm}. In machine learning, set estimation has been viewed  as a
classification problem where we have at our disposal only positive examples. 
An interesting discussion on the relation  between density level set estimation, binary classification  and
anomaly detection is given in~\cite{sthusc05}.   In this context,  some algorithms inspired by Support Vector Machine
(SVM) have been studied  in \cite{shplshsmwi01,sthusc05,veve06}. A kernel method based on kernel principal component analysis is
presented in  \cite{hof07} and is essentially a special case of our framework.\\
{\bf Manifold Learning}. As we mentioned before, a setting which is  of
special interest is the one in which $X$ is $\R^d$ and $X_\rho$ is a 
low dimensional Riemannian submanifold. In this case, the error of an estimator is  
studied in terms  of the error functional
$$
d_\rho(X_\rho, X_n)=\int_{X_\rho} d(x, X_n)d\rho(x),
$$
where $d$ is the Euclidean metric.
Some results in this framework are given in  {\cite{alchma11,mapo10,nami10}}.\\
{\bf Computational Geometry}. 
A classic situation, considered for example in  image reconstruction problems, is when 
the set $X_\rho$ is a hyper-surface of  $\R^d$ and the data $x_1,\ldots,x_n$ are either 
 chosen deterministically or sampled uniformly. The goal in this case is to find  a smooth 
 function $f$ that gives the Cartesian equation of the hyper-surface, see for example {\cite{PoissonSurface,Eigencrust,mls}}.

\section{Kernels, Integral Operators and  Geometry in Spaces of Probabilities}\label{sec:basic}


In this section we establish  the results that  provide the foundations of our approach.
The basic framework in this paper is described by a triple $(X,\rho,K)$, where
\begin{itemize}
\item[-] $X$ is a set  (endowed with a $\sigma$-algebra $\A{X}$);
\item[-] $\rho$ is a probability measure defined on  $X$;
\item[-] $K$ is a {(real)} reproducing kernel on $X$, i.e.~a {real} function on $X\times X$ of positive type.
\end{itemize}

We interpret $X$ as the data space and  $\rho$ as the probability distribution generating the data. 
Roughly speaking, the kernel $K$ provides a natural {\em similarity
measure} on $X$ and  {it defines} its geometry.

We denote by $\hh$ the reproducing kernel Hilbert {space} associated
with the reproducing kernel $K$ (we refer to \cite{aron50,stch08} for
an exhaustive review on the theory of reproducing kernel Hilbert
spaces). The scalar product and norm in $\hh$ are denoted by
$\scal{\cdot}{\cdot}$ and $\nor{\cdot}$, respectively. We recall that
the elements of $\hh$ are {real} functions on $X$, and the
reproducing property $f(x) = \scal{f}{K_x}$ holds true for all $x\in
X$ and $f\in\hh$, where $K_x\in\hh$ is defined by $K_x(\yps)=K(\yps,
x)$. 


In order to prove our results, we need some technical conditions on $K$.
  \begin{assm}\label{A}
The kernel $K$ has  the following properties:
  \begin{enumerate}[a)]
  \item\label{A3} for all $x,\yps\in X$ with $x\neq \yps$ we have $K_x\neq K_{\yps}$;
  \item\label{A1} the associated reproducing kernel Hilbert space $\hh$ is separable;
  \item\label{A2}  the {real} function $K$ is measurable with respect to the product $\sigma$-algebra
$\A{X}\otimes\A{X}$;
  \item\label{A4} for all $x\in X$, $K(x,x)=1$.
  \end{enumerate}
\end{assm}

Assumptions \ref{A}.\ref{A3}), \ref{A}.\ref{A1}) and \ref{A}.\ref{A2}) are minimal requirements. In particular, Assumptions \ref{A}.\ref{A3}) and \ref{A}.\ref{A1}) are needed in order to define a separable metric structure on $X$, while Assumption \ref{A}.\ref{A2}) ensures that such metric topology is compatible with the $\sigma$-algebra $\A{X}$ (see~Proposition~\ref{metrica} below). In Proposition~\ref{Prop.supp.}, the combination of~\ref{A}.\ref{A3}), \ref{A}.\ref{A1}) and \ref{A}.\ref{A2}) will allow us to define the support $X_\rho$ of the probability measure $\rho$, as anticipated in Section \ref{sec:summary}. Assumption~\ref{A}.\ref{A4}), instead, is a normalization requirement, and could be replaced by a suitable boundedness condition 
(in fact, even weaker integrability conditions could also be considered). We choose the normalization $K(x,x) = 1$ $\forall x\in X$ since it makes 
equations more readable, and it is not restrictive in view of {Proposition~\ref{l:normalization} in \ref{app:lemma}.}  

We now show how {the above assumptions} {allow} us to
define a metric on $X$ and to characterize the corresponding support
of $\rho$ in terms of the integral operator with  kernel  $K$.

\subsection{Metric induced by a kernel}

Our first result makes $X$ a separable metric space isometrically
embedded in $\hh$. This point of view is developed in
\cite{smazho08}.  The relation between metric spaces  isometrically
embedded in Hilbert spaces and kernels of positive type was studied by
Schoenberg around 1940. A recent discussion on this topic can be found in Chapter~2~\S~3 of \cite{BeChRe84}.
\begin{prop}\label{metrica}
Under~Assumption~\ref{A}.\ref{A3}), the map $\dk:X\times X\to [0,+\infty[$ defined by 
\begin{equation}
\dk(x,\yps)=\nor{K_x-K_{\yps}}=\sqrt{K(x,x)+K(\yps,\yps)-{2}K(x,\yps)}\label{metric}
\end{equation}  
is a metric on $X$. Furthermore
\begin{enumerate}[i)]
\item\label{P.1.i} the map  {$x\mapsto K_x$ is an isometry from
    $X$ into $\hh$};
\item\label{P.1.ii} {the kernel} $K$ is a continuous function on $X\times X$, and each $f\in\hh$ is a continuous function.
\end{enumerate}
If {also} Assumption~\ref{A}.\ref{A1}) is satisfied, then
\begin{enumerate}[i)]
\setcounter{enumi}{2}
\item\label{P.1.mio} the metric space $(X,\dk) $ is separable.
\end{enumerate}
Finally, if also Assumption~\ref{A}.\ref{A2}) holds true, then
\begin{enumerate}[i)]
\setcounter{enumi}{3}
\item\label{P.1.iii} the closed subsets of $X$ are measurable (with respect to
  $\A{X}$); 
\item\label{P.1.iv}  if $Y$ is a topological space endowed with its Borel
  $\sigma$-algebra and $f: X\to Y$ is continuous, then $f$ is
  measurable; in particular, the functions in  $\hh$ are measurable.
\end{enumerate}
\end{prop}

\begin{proof} Many of these properties are known in the literature, see for example~\cite{cadeto06,stch08} and references therein. For the reader's convenience, we give a self-contained short proof.\\
Assumption~\ref{A}.\ref{A3}) states that { the map $x\mapsto K_x$} is injective. Since
{$\dk(x,\yps)=\nor{K_x-K_y}$} by definition, $\dk$ is the
  metric on $X$ making {$x\mapsto K_x$ an isometry}, as claimed in item \ref{P.1.i}).
{About~\ref{P.1.ii}), the kernel $K$ is continuous since $K(x,\yps) = \scal{K_\yps}{K_x}$ and {the map $x\mapsto K_x$ is continuous by item \ref{P.1.i}); furthermore,} the elements of $\hh$ are continuous by the reproducing property $f(x) = \scal{f}{K_x}$.}\\
If also Assumption~\ref{A}.\ref{A1}) holds true, then the set
${\set{K_x\mid x\in X}}$ is separable {in $\hh$, and so is $X$ as the map $x\mapsto K_x$ is isometric from $X$ onto $\set{K_x\mid x\in X}$.} Item \ref{P.1.mio}) then follows.\\
Suppose now that also Assumption~\ref{A}.\ref{A2}) holds true. Then the map $\dk$ is a
measurable map,  so that the open balls of $X$ are measurable.
Since $X$ is separable, any open set is a countable union of open balls, hence it is measurable. It follows that the closed subsets are
measurable, too, hence item \ref{P.1.iii}).\\
Let $Y$ and $f$ be as in item \ref{P.1.iv}). If $A\subset Y$ is closed, then
$f^{-1} (A)$ is closed in $X$, hence measurable by item \ref{P.1.iii}). It
follows that $f^{-1} (A)$ is measurable for all Borel sets $A\subset
Y$, i.e.~$f$ is measurable. Since the elements of $\hh$ are
continuous by \ref{P.1.ii}), they are measurable, and item \ref{P.1.iv}) is proved.
\end{proof}

In the rest of the paper we will always consider $X$ as a topological
metric space with metric $\dk$. Note that $\dk$ is  the metric induced
on $X$ by the norm of $\hh$ through the embedding ${x\mapsto K_x}$. The next result shows that under our assumptions we can define the set $X_\rho$ as the smallest closed subset of $X$ having measure one.

\begin{prop}\label{Prop.supp.}
Under Assumptions~\ref{A}.\ref{A3}), \ref{A}.\ref{A1}) and \ref{A}.\ref{A2}), there exists a unique closed subset  $X_\rho\subset X$ with $\rho(X_\rho)=1$ satisfying the following property: if $C$ is a closed subset  of $X$ and  $\rho(C)=1$, then $C\supset X_\rho$.
\end{prop}
\begin{proof}
Define the measurable set $X_\rho$ as
\begin{equation*}
  \label{support}
  X_\rho=\bigcap_{
    \begin{smallmatrix}
      C\text{ \rm closed }\\\rho(C)=1
    \end{smallmatrix}} C .
\end{equation*}
Clearly,  $X_\rho$ is closed {and  measurable by Proposition \ref{metrica}}. Since $X$ is separable, there exists a
  sequence of closed subsets $(C_j)_{j\geq 1}$ such that
  every closed subset $C=\cap C_{j_k}$, for some suitable subsequence.
Hence, $ X_\rho=\bigcap\limits_{ j\mid \rho(C_j)=1} C_j$ and, as a consequence, 
$\rho(X_\rho)=1$. 
\end{proof}
We add one remark. The set $X_\rho$ is called {\em the support} of the measure 
$\rho$ and clearly depends both on the probability distribution and on the topology  induced by the kernel
$K$  through  the metric $\dk$ on $X$.

\subsection{Separating Kernels}
The following definition of separating kernel plays a central role in our approach.
\begin{defn}\label{separated} We say that {the reproducing kernel Hilbert space $\hh$} separates a subset $C\subset X$,  if, for all $x\not\in C$, there exists $f\in\hh$ such that
\begin{equation}
f(x) \neq 0 \quad\text{ and }\quad f(\yps)= 0 \quad \forall \yps\in C . 
\label{separa}
\end{equation}
In this case we also say that the corresponding reproducing kernel separates $C$.
\end{defn}
We add some comments. First, in~\eqref{separa} the function $f$ depends on
$x$ and $C$. Second, the reproducing property and \eqref{separa} imply that
$K_x\neq 0$ and $K_x\neq K_{\yps}$ for all $x\not\in C$ and $\yps\in
C$ (compare with Assumption~\ref{A}.\ref{A3})).
Finally, we stress that a different notion of {\em separating} property is
given in \cite{stch08}.
\begin{rem}\label{sep_lin}
Given an arbitrary reproducing kernel Hilbert space $\hh$, there exist sets
that are not separated by $\hh$. For example, if $X= \R^d$ and
$\hh$ is the reproducing kernel Hilbert  space with linear kernel
$K(x,\yps)= x^T \yps$, the only sets separated by $\hh$ are the linear
manifolds, that is, the set of points defined by homogeneous linear equations (see Figure \ref{separability}). A natural question is then 
whether  there exist kernels capable of separating large classes of subsets and in particular  all
the closed subsets. Section \ref{RKHSCR}  anwers positively to this question, introducing the notion of completely separating kernels.
\end{rem}

Next, we provide an equivalent characterization of the separating property, which will be the key 
to a computational approach to support estimation. 
For any set $C$, let $P_C:\hh\to\hh$ be the orthogonal projection onto
the closed subspace 
$$
\hh_C=\overline{\text{span}\set{ K_x\mid x\in C}},
$$
i.e.~the closure of the linear space generated by the family
  $\set{K_x\mid x\in C}$. Note that  $P_C^2=P_C$, $P_C^{\top}=P_C$ and
\begin{equation*}\label{proj_def} 
\ker{P_C}=\set{K_x\mid x\in
  C}^{\perp}=\{ f\in\hh \mid f(x) = 0 \ \forall x\in C \}. 
  \end{equation*}  
  Moreover, define the function
\begin{equation}\label{FC_def}
F_C : X\to\R, \qquad F_C (x) = \scal{P_CK_x}{K_x}.
\end{equation}
{\begin{rem}\label{rem:Stein1}
The Hilbert space $\hh_C$ is a closed subspace of the reproducing kernel
Hilbert space $\hh$, and  it is itself a reproducing kernel Hilbert
space of functions on $X$ with reproducing kernel  $K_C(x,\yps)=\scal{P_CK_\yps}{P_CK_x}=\scal{P_CK_\yps}{K_x}$. {Note that $K_C(x,\yps) = K(x,\yps)$ for all $x,y\in C$ by definition of $P_C$.} Clearly, the function $F_C$ corresponds to  the value of $K_C$ on the diagonal.
\end{rem}
}
Then,  we have the following theorem.
\begin{thm}\label{prop_proj} 
For any subset $C\subset X$, the following facts are equivalent:
\begin{enumerate}[i)]
\item  $\hh$ separates the set $C$ ;
\item for all $x\not\in C$, $K_x \notin \ran{P_C}$;
\item $\displaystyle{ C=\set{x\in X \mid F_C(x)=K(x,x)}}$;
\item $\displaystyle{ {\set{K_x\mid x\in C} = \set{K_x\mid x\in X}}\cap \ran{P_C}}$.
\end{enumerate}
Under Assumption~$\ref{A}.\ref{A3})$, 
if $C$ is separated by $\hh$, then $C$ is closed with respect to the metric~$\dk$.
\end{thm}

\begin{proof} We first prove that
i) $\Rightarrow$ ii). Given $x\notin C$,  by assumption there is $f\in \hh$ such that $\scal{f}{K_x}=f(x)\neq 0$, i.e.~$K_x\not\in \set{f}^\perp$, and $\scal{f}{K_{\yps}} =f(\yps)=0$ for all $\yps\in C$, i.e.~$f\in \ker{P_C} = \ran{P_C}^\perp$.
It follows that $\ran{P_C} \subset \set{f}^\perp$, and then $K_x\notin \ran{P_C}$.\\
We prove ii) $\Rightarrow$ iii). {If $x\in C$, then $K_x\in\ran{P_C}$ by definition of $P_C$, so that
$F_C(x)=K(x,x)$. Hence $C\subset \set{x\in X \mid F_C(x) = K(x,x)}$. If $x\not\in C$, then by assumption $P_C K_x \neq K_x$, i.e.~$(I-P_C)K_x\neq
0$. By the equality
$$
\nor{(I-P_C)K_x}^2 = \scal{K_x}{K_x} - \scal{P_C K_x}{K_x} - \scal{K_x}{P_C K_x} + \scal{P_C K_x}{P_C K_x} = K(x,x)-F_C(x) ,
$$
this implies $F_C(x) \neq K(x,x)$. Hence $C\supset \set{x\in X \mid F_C(x) = K(x,x)}$.\\}
We prove iii) $\Rightarrow$ i). If $x\not\in C$, define
$f=(I-P_C)K_x\in\ker{P_C}$, so that $f(\yps)=0$ for all $\yps\in
C$. Furthermore, $f(x)=K(x,x)-F_C(x)\neq 0$. Thus, $f$ separates the set $C$.\\
Finally, iv) is
a restatement of~ii) taking into account that $K_x\in\ran{P_C}$ for
all $x\in C$ by construction.\\
Under  Assumption~$\ref{A}.\ref{A3})$, {the map $x\mapsto F_C(x)-K(x,x) = \scal{P_C K_x}{K_x} - 
K(x,x)$ is continuous by Proposition~\ref{metrica}.} By item iii), $C$
is the $0$-level set of this function, hence $C$ is closed.
\end{proof}

{Proposition \ref{l:normalization} in \ref{app:lemma} shows} that the reproducing kernel $K$ can be normalized 
under the mild assumption that $K(x,x)\neq 0$ for all $x\in X$, so that Assumption~\ref{A}.\ref{A4}) can be satisfied up to a rescaling of $K$.

\subsubsection{A Special Case: Metric Spaces}\label{sec:special_case}

It may be the case that the set $X$ has its own metric $d_X$, and the
$\sigma$-algebra $\A{X}$ is  the Borel $\sigma$-algebra
associated with the topology induced by $d_X$.  The following
proposition shows 
that {the metrics $\dk$ and $d_X$ induce the same topology on $X$}, provided that $\hh$ separates all the $d_X$-closed subsets
and the corresponding kernel is continuous. 
\begin{prop}\label{c:connection}
Let $X$ be a separable metric space with respect to a metric
$d_X$, and $\A{X}$ the corresponding Borel $\sigma$-algebra. Let $\hh$ be a reproducing kernel Hilbert space on $X$ with kernel $K$.
Assume that  the kernel $K$ is a continuous function with respect to $d_X$ and that
the space $\hh$ separates every subset of $X$ which is closed with respect to $d_X$.
Then
\begin{enumerate}[i)]
\item Assumptions~$\ref{A}.\ref{A3})$, $\ref{A}.\ref{A1})$ and $\ref{A}.\ref{A2})$ hold true, and $K(x,x)>0$ for all
  $x\in X$;
\item a set is closed with respect to $\dk$ if and only if it is
  closed with respect to $d_X$. 
\end{enumerate}
\end{prop}
\begin{proof}
The kernel is measurable and the space $\hh$ is separable  by Proposition
$5.1$ and Corollary $5.2$ in \cite{cadeto06}.   Since  the points are closed sets for $d_X$ and the $d_X$-closed sets are separated by $\hh$, then $K_x\neq 0$ (i.e.~$K(x,x) > 0$) for all $x\in X$ and $K_x\neq K_{\yps}$ if $x\neq\yps$ by the discussion following Definition \ref{separated}.  \\
We show that $d_X$ and $\dk$ are equivalent metrics. Take a
  sequence  $(x_j)_{j\geq 1}$ such that for some $x\in X$ it holds that $\lim_{j\to\infty} d_X(x_j,x)=0$. 
Since $K$ is
continuous with respect to $d_X$, we have $\lim_{j\to\infty} \dk(x_j,x)=0$. Hence, the  $\dk$-closed
sets are  $d_X$-closed, too.  Conversely, if the set $C$ is {$d_X$-closed},
since $\hh$ separates $C$, Theorem~\ref{prop_proj} implies that $C=\set{x\in 
  X\mid K(x,x)-F_C(x)=0}$, which is a $\dk$-closed set by $\dk$-continuity of the map $x\mapsto K(x,x)-F_C(x)$.
\end{proof}
Item ii) of the above proposition states that the metrics $\dk$ and $d_X$ are equivalent and
implies that the set $X_\rho$ defined in Proposition~\ref{Prop.supp.} coincides with  the support of $\rho$ 
with respect to the topology induced by $d_X$.

\subsection{{The} Integral Operator Defined by the Kernel}\label{sec:int}

We denote by $\mathcal S_1$ the Banach space of the trace class operators on $\hh$, with trace class norm
$$
\nor{A}_{\mathcal S_1} = \tr{(A^{\top} A)^{\frac{1}{2}}} = \sum_{i\in I} \scal{(A^{\top} A)^{\frac{1}{2}} e_i}{e_i} ,
$$
where $\set{e_i}_{i\in I}$ is any orthonormal basis of $\hh$. Furthermore, we let $\mathcal S_2$ be the separable Hilbert space of the Hilbert-Schmidt operators on $\hh$, with Hilbert-Schmidt
norm
\[
\nor{A}^2_{\mathcal S_2}=\tr{A^{\top}A} = \sum_{i\in I} \nor{A e_i}^2 .
\]
Finally, if $A$ is any bounded operator on $\hh$, we denote by $\nor{A}_\infty$ its uniform operator norm. It is standard that $\nor{A}_\infty\leq \nor{A}_{\mathcal S_2} \leq \nor{A}_{\mathcal S_1}$. Moreover, for all functions $f_1,f_2\in \hh$, the rank-one operator $f_1\otimes f_2$ on $\hh$ defined by
$$
(f_1\otimes f_2)(f) = \scal{f}{f_2}\, f_1 \qquad \forall f\in\hh 
$$
is trace class, and $\nor{f_1\otimes f_2}_{\mathcal S_1} = \nor{f_1\otimes f_2}_{\mathcal S_2} = \nor{f_1}\nor{f_2}$.

We recall a few facts on integral operators with kernel $K$ (see  \cite{cadeto06} for  proofs  and further discussions).
Under Assumption~\ref{A}, the $\mathcal S_1$-valued map $x\mapsto  K_x\otimes
K_x$ is Bochner-integrable with respect to $\rho$, and its integral
\begin{equation}\label{eq:2}
T {=} \int_{X} K_x\otimes K_x d\rho (x)
\end{equation}
defines a positive trace class operator $T$ with $\nor{T}_{\mathcal S_1} = \tr{T} = 1$ (a short proof is given in Proposition \ref{Tprop} of the Appendix).
Using the reproducing property of $\hh$, it is straightforward to see that 
$T$ is simply  the integral operator with kernel $K$ acting on $\hh$, i.e.
\[
(Tf)(x)=\int_X K(x,\yps) f(\yps) d\rho(\yps) \qquad \forall f\in\hh .
\]

The following is a key result in  our approach. 
\begin{thm}\label{T}
Under Assumption~\ref{A}, the null space of $T$ is
\begin{equation}\label{nucleo}
\ker{T}=\set{K_x\mid x\in X_\rho}^\perp =\ker{P_{X_\rho}},
\end{equation}
where $X_\rho$  {is the support of $\rho$ as defined in Proposition \ref{Prop.supp.}}.
\end{thm}

\begin{proof}
Note that , for all $f\in \hh$,  the set
\[
C_f=\set{x\in X\mid f(x)=0}= \set{x\in X\mid  \scal{f}{K_x}=0}
\]
is closed since $f$ is continuous. We now prove Equation
\eqref{nucleo}. Since $T$ is a positive operator, spectral theorem gives
that $Tf=0$ if and only if $\scal{Tf}{f}=0$.  The definition of $T$
and the reproducing property gives that
\begin{align*}
  \scal{Tf}{f}=\int_X \scal{(K_x\otimes K_x)f}{f} d\rho (x) =
\int_X \abs{\scal{ K_x}{f}}^2 d\rho (x) =\int_X |f(x)|^2 d\rho (x),
\end{align*}
hence the condition $\scal{Tf}{f}=0$  is equivalent to the fact that
$f(x)=0$ for 
$\rho$-almost every $x\in X$. Hence $f\in\ker{T}$ if and only if $\rho(C_f) = 1$, i.e.~$C_f\supset X_\rho$, or equivalently $\scal{f}{K_x}=0$ $\forall x\in X_\rho$. Equation \eqref{nucleo} then follows.
\end{proof}

In the following, we will use the abbreviated notation $P_\rho = P_{X_\rho}$. Note that the space $\hh$ splits into the direct sum $\hh=\hh_\rho\oplus \hh_\rho^\perp$, where
\begin{align*}
\hh_\rho & = {\ran{P_\rho} =\overline{\ran{T}} =\overline{\text{span}\{K_x \mid x\in X_\rho
\}}} \\
\hh_\rho^\perp & = \ker{P_\rho} =\ker{T}=\{ f\in\hh \mid f(x) = 0 \ \forall x\in X_\rho \}.
\end{align*}

{\begin{rem}\label{rem:Stein2}
The reproducing kernel Hilbert space $\hh_\rho$ (see Remark \ref{rem:Stein1}) 
has been considered  before \cite{stsc12},  and in particular in the context of semi-supervised manifold regularization \cite{benisi06}, where
$X_\rho$ is assumed to be an embedded manifold. The corresponding 
reproducing kernel is $K_\rho (x,y) = \scal{P_\rho K_y}{K_x}$ and  $F_{X_\rho}(x) = K_\rho (x,x)$. See also the discussion in Section~6.
\end{rem}}

Under Assumption~\ref{A}, we also introduce the integral operator $L_K : L^2(X,\rho)\to L^2(X,\rho)$,
\begin{equation*}
(L_K \phi) (x) {=} \int_X K(x,\yps) \phi (\yps) d\rho (\yps) \qquad
\forall \phi\in L^2(X,\rho),\label{eq:7}
\end{equation*}
which is  a positive trace class operator, too. Note the difference
between the operators $T$ and $L_K$: although their definitions are
formally the same, the respective domains and images change. 

{Since $T$ and $L_K$ are positive trace class operators, 
by the Hilbert-Schmidt theorem each of them admits an
orthonormal family of eigenvectors in $\hh$ and $L^2(X,\rho)$,
respectively,  with a corresponding family of positive eigenvalues.
The two spectral decompositions are strongly related, as we now
briefly recall (see also Proposition~8 of \cite{RoBeDe10} and Theorem 2.11 of \cite{stsc12}).}

{Denote by $(\sigma_j)_{j\in J}$ the (finite or countable) family of
strictly positive eigenvalues of $L_K$, where each eigenvalue is
repeated according to its (finite) multiplicity. For each $j\in J$
select a corresponding eigenvector $\phi_j\in L^2(X,\rho)$ {in such a
way that the sequence $(\phi_j)_{j\in J}$ is orthonormal in $L^2(X,\rho)$.}
Hilbert-Schmidt theorem provides that
\begin{equation}\label{eq:3}
L_K=\sum_{j\in J} \sigma_j \phi_j\otimes \phi_j ,
\end{equation}
where the series converges in trace norm.  In general, each
element $\phi_j$ is an equivalence class of functions defined
$\rho$-almost everywhere. In particular, the value of $\phi_j$
is not defined outside $X_\rho$. However,  in each equivalence  class
we can choose  a unique continuous function, denoted again by
$\phi_j$, which is defined at every point of $X$ by means of 
the {\em extension equation} {\cite{coilaf06,RoBeDe10}}
\begin{equation}\label{eq:outofsample}
\phi_j(x) = \sigma_j^{-1} \int_X K(x,y) \phi_j (y)  d\rho (y)\qquad \forall x\in X .
\end{equation}
With this choice, which will be implicitly assumed in the following,
the family $(\sigma_j)_{j\in J}$ 
coincides with the family of strictly positive eigenvalues of 
$T$ (with the same multiplicities), $(\sqrt{\sigma_j}  \phi_j)_{j\in
  J}$ is a orthonormal
family in $\hh$ of  eigenfunctions of $T$,  and 
\begin{equation}
T =\sum_{j\in J} \sigma_j\,\, (\sqrt{\sigma_j}  \phi_j)\otimes
(\sqrt{\sigma_j}   \phi_j) = \sum_{j\in J} \sigma_j^2 \ \phi_j\otimes
\phi_j,\label{eq:6} 
\end{equation}
where  the series converges
in the Banach space $\mathcal S_1$ (hence in $\mathcal S_2$), see
e.g. \cite{cadeto06,RoBeDe10,stsc12}. 
As $\nor{T}_{\mathcal S_1} = 1$, the positive
sequence $(\sigma_j)_{j\in J}$ is summable and sums up to
$1$.  It is clear that the family $(\sqrt{\sigma_j}
  \phi_j)_{j\in J}$ is an orthonormal basis of the Hilbert space
  $\hh_\rho$. Conversely, let
$(f_j)_{j\in J}$  be an orthonormal basis  of $\hh_\rho$ of
eigenvectors of $T$ with corresponding eigenvalues $(\sigma_j)_{j\in
  J}$. Define 
\[
\phi_j (x) = \sigma_j^{-\frac{1}{2}} f_j (x) \qquad  \forall x\in X {.}
\]
{Then, it is not difficult to show} that \eqref{eq:3}, \eqref{eq:outofsample} and \eqref{eq:6} hold true.}

\subsection{An Analytic Characterization of the Support}

Let Assumption~\ref{A} hold true. Collecting the previous results, if  $\hh$ separates $X_\rho$, then Theorem \ref{prop_proj} gives that
$$
X_{\rho}=\set{x\in X\mid F_{X_\rho}(x)=1 } .
$$
The function {$F_\rho = F_{X_\rho}$} is defined by \eqref{FC_def} in terms of the projection $P_\rho$, which, in light of Theorem \ref{T}, can be characterized using the operator $T$. 
Indeed, from the  definition of {$F_\rho$} and~\eqref{nucleo} we have
\begin{equation}\label{Projection}
{F_\rho(x)=} \scal{P_\rho K_x}{K_x}=\scal{ T^\dag TK_x}{K_x}=\scal{\theta(T)K_x}{K_x}=\sum_{j\in J} {\sigma_j\abs{\phi_j(x)}^2}
\end{equation}
where $T^\dag$ is the pseudo-inverse of $T$ and $\theta$ is the Heaviside function $\theta (\sigma) = \ind_{]0 , +\infty [} (\sigma)$ (note that with our definition
$\theta(0)=0)$. The above discussion is summarized in the following theorem.

\begin{thm}\label{primo}
{If $\hh$ satisfies
 Assumption~$\ref{A}$ and separates}  the support $X_\rho$ of the
 measure $\rho$, then
\[ X_{\rho}=
\set{x\in X\mid  F_\rho(x)=1}=\set{x\in X\mid \scal{T^\dag TK_x}{K_x}=1 }.\] 
\end{thm}
As we discussed before,  a natural question 
is whether  there exist  kernels capable to separate  {\em all} 
possible closed subsets of $X$. In a learning scenario, 
 this can be translated into a {\em universality} property, in the sense 
that it allows to describe {\em any} probability distribution and
learn consistently its support \cite{degylu96}. 
Note that in a supervised learning framework a similar role is played  by the so called universal kernels 
\cite{cadeto10,Steinwart02}. The following section answers positively to the previous question, introducing and 
studying  the concept of   completely separating kernels. Interestingly,  there are universal
kernels in the sense of \cite{cadeto10,Steinwart02} which do not separate all closed subsets of $X$, as for example the Gaussian kernel.

\section{Completely separating reproducing kernel Hilbert spaces}\label{RKHSCR}

The property defining the class of kernels we are interested {in} is captured by the following definition.
\begin{defn}[Completely Separating Kernel]\label{ass1}
A reproducing kernel Hilbert space $\hh$ satisfying Assumption~$\ref{A}.\ref{A3})$  is called  {\em completely separating}
if $\hh$ separates all the subsets $C\subset X$ which are closed with
respect to the metric~$\dk$ defined by~\eqref{metric}. In this case, we also say that the corresponding reproducing kernel is completely separating.
\end{defn}

The definition of completely separating reproducing kernel Hilbert spaces should be compared with the analogous notion of complete regularity for topological spaces. Indeed, we recall that a  topological space is called {\em completely regular} if, for any closed subset $C$ and  any point $x\notin C$, there exists a continuous function $f$ such that $f(x)\neq 0$ and $f(\yps)=0$ for all $\yps\in C$.
As we discuss below, completely separating reproducing kernels do exist. 
For example, for $X=\R^d$ both the Abel kernel $K(x,\yps)=e^{-\nor{x-\yps}/\sigma}$ and the
$\ell_1$-exponential kernel $K(x,\yps)=e^{-\nor{x-\yps}_1/\sigma}$ are completely separating, where $\nor{x}$ is just the Euclidean 
norm of $x=(x^1, \dots, x^d)$ in  $\R^d$ and $\nor{x}_1 = \sum_{j=1}^d
|x_j|$ is the $\ell_1$-norm.
Indeed this follows from Theorem ~\ref{ale} and Proposition~\ref{prodotto} below, which  give sufficient conditions for a kernel to be completely separating in the case $X=\R^d$.
Note that the Gaussian kernel $K(x,\yps)=e^{-\nor{x-\yps}^2/\sigma^2}$ on
$\R^d$ is not completely separating. This is a consequence of the following fact. 
It is known that the elements of the corresponding reproducing kernel Hilbert space $\hh$ are
analytic functions, see Corollary~4.44 in \cite{stch08}. If $C$ is a closed subset of $\R^d$ with non-empty interior and $f\in\hh$ is equal to zero on $C$, then a standard result in  complex analysis implies that $f(x)=0$ for all
$x\in\R^d$, hence $\hh$ does not separate $C$.

We end this section with {Proposition \ref{prodotto}, which gives} a simple way to build completely separating kernels in high dimensional spaces from completely separating kernels in one dimension, {the latter usually being} easier to characterize.

\subsection{Separating Properties of  Translation Invariant Kernels}

The first result studies translation invariant kernels on $\R^d$, {i.e.~}of the form $K(x,\yps) = K(x-\yps)$. We show that if the
Fourier transform of the kernel satisfies a suitable growth condition,
then the corresponding reproducing kernel Hilbert space is completely
separating. {As usual, $C(\R^d)$ denotes the space of real continuous functions on $\R^d$ and, for any $p\in[1,+\infty\,[$, 
$L^p(\R^d)$ is the space of (equivalence classes of)
real functions on $\R^d$ which are $p$-integrable with respect to the
Lebesgue measure $dx$. We will consider the {\em real} spaces
  $L^p_h(\R^d)$ of hermitian complex functions, {i.e.}
$$
L^p_h(\R^d) = \{\phi_1 + i\phi_2 \mid \phi_1,\phi_2\in L^p(\R^d) \mbox{ and } \phi_1(-x) = \phi_1(x) \, , \, \phi_2(-x) = -\phi_2(x)\} .
$$
If $\phi\in L^1(\R^d)$, its Fourier transform
is the complex hermitian bounded continuous function $\hat{\phi}$ on $\R^d$ given by
$$
\hat{\phi} (z) = \int_{\R^d} e^{-2\pi i z\cdot x} \phi(x) d x .
$$
If $\phi\in L^2(\R^d)$, we denote by {$\hat{\phi}\in L^2_h (\R^d)$} its
Fourier-Plancherel transform, obtained extending the above definition
{on functions $\phi\in L^1(\R^d)\cap L^2(\R^d)$ to a unitary map $L^2(\R^d) \ni \phi \to \hat{\phi}\in L^2_h (\R^d)$.}}

Throughout, we assume $\R^d$ to be a metric space with respect to the
standard metric $d_{\R^d}$ induced by the Euclidean norm.

We need a preliminary result characterizing a reproducing
kernel Hilbert space, whose reproducing kernel is continuous  and
integrable, as a suitable non-closed
subspace of $L^2(\R^d)$. The first part is a converse of
Bochner's theorem (Theorem 4.18 in \cite{fol95}).

\begin{prop}\label{Prop. HK con K di tipo pos.}
Let {$K$ be a} continuous function in  $L^1(\R^d)$ such that
its Fourier transform $\hat{K}$ is strictly positive.  Then the kernel $K(x,\yps)=K(x-\yps)$ is 
positive definite and  its corresponding {(real)} reproducing kernel Hilbert
space $\hh$ is
\begin{equation}
{\hh = \left\{ \phi\in C(\R^d)\cap L^2(\R^d)\mid \int_{\R^d}
  \hat{K} (z)^{-1} |\hat{\phi}(z)|^2 dz < +\infty
\right\} }\label{HK con K L1}
\end{equation}
with norm
\begin{equation}\label{HK con K L1 bis}
\nor{\phi}^2 = \int_{\R^d}\hat{K}(z)^{-1}  |\hat{\phi}(z)|^2 dz \qquad \forall \phi\in \hh .
\end{equation}
\end{prop}
\begin{proof}
{The integral operator
$$
({L_K} \phi)(x)=\int_{\R^d}K(x-\yps) \phi(\yps)\,d\yps=(K*\phi)(x),
$$
is  well defined and bounded from $L^2(\R^d)$ into $L^2(\R^d)$} since $K\in L^1(\R^d)$.  Since ${L_K}$ is a
convolution operator, Fourier transform 
turns it into the operator of multiplication by the bounded function
$\hat{K}$, that is $\widehat{{L_K} \phi} = \hat{K} \hat{\phi}$ for all
{$\phi\in L^2(\R ^d)$}.  It follows that 
$$
\scal{{L_K} \phi}{\phi}_{L^2}= \scal{\hat{K}\hat{\phi}}{\hat{\phi}}_{L^2} {> 0\qquad
\forall \phi\in L^2(\R^d)\setminus\{0\}}
$$
since {$\hat{K}> 0$} by assumption, hence ${L_K}$ is a {strictly} positive
operator. In order to show that $K$ is positive definite, pick a Dirac
sequence $( \pphi_n )_{n\geq 1}$
{as in Chapter VIII.3 of \cite{la93}}, and, for each
$x\in X$, define $\pphi_n^x $ be equal to $\pphi_n^x(\yps)= \pphi_n (\yps- x)$. Fixed
$x_1,x_2,\ldots,x_N\in \R^d$ and $c_1,c_2,\ldots,c_N\in {\R}$, set $\phi_n = \sum_{i=1}^N c_i \pphi_n^{x_i}$, then
$$
0 \leq\scal{{L_K} \phi_n}{\phi_n}_{L^2} = \sum_{i,j=1}^N c_i {c_j} \scal{{L_K} \pphi_n^{x_i}}{\pphi_n^{x_j}}_{L^2}
\mathop{\frecc}_{n\to\infty} \, \sum_{i,j=1}^N c_i {c_j} K(x_j,x_i),
$$
where the last equality is due to continuity of $K$ and the usual properties of Dirac sequences. It follows that $\sum_{i,j=1}^N c_i {c_j} K(x_j,x_i)\geq 0$, i.e.~the kernel $K$ is positive definite.\\
Let $\hh$ be {the (real)} reproducing kernel Hilbert space associated to $K$. Since
the support of the Lebesgue measure is $\R^d$,  {Mercer theorem (as stated e.g.~in Proposition 6.1 of \cite{cadeto06} and the subsequent
  discussion, or Theorem 2.11 of \cite{stsc12}) shows that ${L_K}^{1/2}$ is a unitary
isomorphism from $L^2 (\R^d)$ onto $\hh$. More precisely, for any
$\psi\in L^2 (\R^d)$ there exists a unique function $\phi\in C(\R^d)$
such that its equivalence class belongs to ${L_K}^{1/2}\psi\in L^2(\R^d)$, and the
correspondence $\psi\mapsto \phi$ is an {isometry} from $L^2
(\R^d)$ onto $\hh$. By further applying the Fourier-Plancherel
transform and taking into account that $\hat{\phi}(z)=
\sqrt{\hat{K}(z)}\, \hat{\psi}(z)$ for almost all $z\in\R^d$, one has
$$
\nor{\phi}^2_\hh = \nor{\psi}^2_{L^2(\R^d)} = \nor{\hat{\psi}}^2_{L^2_h(\R^d)} = \int_{\R^d}\hat{K}(z)^{-1}  |\hat{\phi}(z)|^2 dz <+\infty ,
$$
so that~\eqref{HK con K L1} and \eqref{HK con K L1 bis} follow.}
\end{proof}
We now state a sufficient condition on $K$ ensuring
that $\hh$ is completely separating. 
\begin{thm}\label{ale}
Let {$K$} be a continuous function in $L^1(\R^d)$ such that 
\begin{equation}
\hat{K}(z) \geq \frac{a}{\left(1+b\nor{z}^{\gamma_1}\right)^{\gamma_2}} \qquad \forall y\in\R^d \label{poli}
\end{equation}
for some $a,b,\gamma_1,\gamma_2 >0$.  Then,
\begin{enumerate}[i)]
\item the translation invariant kernel $K(x,\yps)=K(x-\yps)$ is positive definite and continuous; 
\item the topologies induced by the metric $\dk$ and the Euclidean metric $d_{\R^d}$ coincide on $\R^d$;
\item the kernel $K$ is completely separating.
\end{enumerate}
\end{thm}
\begin{proof}
Condition~\eqref{poli} implies that $\hat{K}$ is strictly positive, so
item~i) follows from Proposition~\ref{Prop. HK con K di tipo pos.}. In particular, from~\eqref{HK con K L1}  we see {that,} {if  $\phi\in L^2 (\R^d)$ and
$\int_{\R^d}\left(1+b\nor{z}^{\gamma_1}\right)^{\gamma_2} |\hat{\phi}(z)|^2 dz$ is finite, 
then $\phi\in\hh$. }
This implies that $C^\infty_c (\R^d) \subset \hh$: indeed,  if $\phi\in C^\infty_c (\R^d)$, then $\hat{\phi}$ is a Schwartz
function on $\R^d$ {(Theorem 3.2 in \cite{stwe71})}, hence the
last integral is convergent. Functions in $C^\infty_c (\R^d)$ separate
every set $C$ which is closed with respect to the metric $d_{\R^d}$
{(as it easily follows by suitably translating and dilating the
  function $\psi\in C^\infty_c (\R^d)$ defined in item (b) p.~19 of
  \cite{stwe71}),} hence $\hh$ separates the $d_{\R^d}$-closed subsets. Items ii) and iii) then follow from Proposition \ref{c:connection}.
\end{proof}
As an application, we show that the Abel kernel is completely separating. 
\begin{prop}\label{Prop. exp sep.}
Let
\begin{equation}\label{AbelKer}
K : \R^d \times \R^d \to \R , \qquad K(x,\yps) = e^{-\frac{\nor{x-\yps}}{\sigma}} ,
\end{equation}
with $\sigma > 0$. Then $K$ is a positive definite kernel and the
corresponding reproducing
kernel Hilbert space $\hh$ is completely separating for all $d\geq 1$.
\end{prop}
\begin{proof}
A standard Fourier transform computation gives
\begin{equation}\label{K^}
\hat{K} (z) = \frac{1}{2\pi\sigma}\pi^{-\frac{d+1}{2}} \Gamma \left(\frac{d+1}{2}\right) \left( \frac {1} {4\pi^2\sigma^2}+ \nor{z}^2 \right)^{-\frac{d+1}{2}} ,
\end{equation}
where $\Gamma$ is Euler gamma function (Theorem~1.14 in \cite{stwe71}). The claim
then follows from Theorem~\ref{ale}. 
\end{proof}
Equations~\eqref{HK con K L1}, \eqref{HK con K L1 bis} and \eqref{K^} show
that (up to a rescaling of the norm) the reproducing kernel Hilbert space associated to the Abel Kernel \eqref{AbelKer} is just 
$W^{(d+1)/2} (\R^d)$,  the Sobolev space of order $(d+1)/2$.

\subsection{Building Separating Kernels}

The following result gives  a way to construct completely separating reproducing
kernel Hilbert spaces on high dimensional spaces. 
\begin{prop}\label{prodotto}
If $X_i$, $i = 1,2\ldots d$, are sets and $K^{(i)}$ are
completely separating reproducing kernels on $X_i$ for all $i = 1,2\ldots d$, then the product kernel
\[K((x_1 , \ldots, x_d),(\yps_1 , \ldots, \yps_d)) {=} K^{(1)} (x_1,\yps_1) \cdots K^{(d)} (x_d,\yps_d)\] 
is completely separating on the
set $X = X_1 \times X_2 \times\ldots\times X_d$. 
\end{prop}
\begin{proof}
Each set $X_i$ and $X$ are endowed with the metric $d_{K^{(i)}}$ and $\dk$ induced by the corresponding kernels, and $\hh_i$ and $\hh$ denote the reproducing kernel Hilbert spaces with kernels $K^{(i)}$ and $K$, respectively. A standard result gives that  $\hh = \hh_1\otimes \ldots \otimes \hh_d$ and $K_x = K^{(1)}_{x_1}\otimes \ldots \otimes K^{(d)}_{x_d}$ for all $x = (x_1 , \ldots, x_d) \in X$  \cite{aron50} .
We claim that
    the $\dk$-topology on $X$ is contained in the product
    topology of the $d_{K^{(i)}}$-topologies on $X_i$ (actually, it is not
    difficult to show that the two topologies coincide). Indeed, if
    $( x_{i,k} )_{k\geq 1}$ are sequences in $X_i$ such 
    that $\lim_{k\to\infty} d_{K^{(i)}}(x_{i,k} , x_i) = 0$ for all $i=1,\ldots, d$, then
\begin{eqnarray*}
  && \lim_{k\to\infty} \dk\left( (x_{1,k} ,\ldots , x_{d,k}) , (x_1 ,\ldots ,
    x_d) \right)^2 = \lim_{k\to\infty} \nor{K_{(x_{1,k} ,\ldots , x_{d,k})} -
    K_{(x_1 ,\ldots , x_d)}}^2 \\ 
  && \qquad \quad = \lim_{k\to\infty} [K^{(1)} (x_{1,k} , x_{1,k}) \cdots K^{(d)} (x_{d,k},
  x_{d,k}) - {2} \, K^{(1)} (x_{1,k} , x_1) \cdots K^{(d)} (x_{d,k},x_d) \\
  && \qquad \qquad + K^{(1)} (x_1 , x_1) \cdots K^{(d)} (x_d,x_d)] \\
  && \qquad \quad =0 ,
\end{eqnarray*}
since $\lim_{k\to\infty} K^{(i)}(x_{i,k} , x_{i,k}) = \lim_{k\to\infty} K^{(i)}(x_{i,k} , x_i) = K^{(i)}(x_i ,
x_i)$. We now prove that $\hh$ is completely separating. If $C\subset X$ is $\dk$-closed and $x = (x_1 , \ldots, x_d) \in X\setminus C$, since $C$ is also closed in the product topology, for all $i=1,\ldots, d$ there exists an open neighborhood $U_i$ of $x_i$ in $X_i$ such that $U = U_1 \times \ldots \times U_d\subset X\setminus C$.   Since each $\hh_i$ is completely separating,  for all $i=1,\ldots,d$ there exists $f_i \in
\hh_i$ such that $f_i (x_i) \neq 0$ and $f_i (\yps_i) = 0$ for all $\yps_i\in X_i \setminus U_i$. Then the product function $f=f_1\otimes \ldots \otimes f_d$ is in $\hh$, and satisfies $f(x) \neq 0$
and $f (\yps) = 0$ for all $\yps\in C$.
\end{proof}
As a consequence, the Abel kernel defined by  the $\ell_1$-norm
\begin{align*} 
K (x,\yps)  = e^{-\frac{\nor{x-\yps}_1}{\sigma}} =\prod_{i=1}^d\, e^{-\frac{\abs{x_i-\yps_i}}{\sigma}}, \qquad x=(x_1,\ldots,x_d), \, \yps=(\yps_1,\ldots,\yps_d)
\end{align*}
is completely separating since each kernel in the product is positive definite and completely 
separating   by Proposition \ref{Prop. exp sep.}.

\section{A Spectral  Approach to Learning the Support}\label{sec:learn}

In this section we study the set estimation problem in the context of
learning theory.  We fix a triple $(X,\rho,K)$ as in Section \ref{sec:basic}, and assume throughout that the reproducing kernel $K$ satisfies 
Assumption~\ref{A}. We regard $X$ as a metric space with respect to $\dk$, and continue to denote by $X_\rho$ the support of $\rho$ defined in Proposition \ref{Prop.supp.}.

If $\hh$ separates $X_\rho$, Theorem~\ref{primo}  shows that the support   $X_\rho$  is the $1$-level set of a suitable function $F_\rho$ defined by the integral   operator $T$, and therefore depending on $K$ and $\rho$. However, the probability   distribution $\rho$ is unknown, as we only have a set of i.i.d.~points $x_1, \dots, x_n$ sampled from $\rho$ at our disposal. Our task is now to use our sample in order to estimate the set $X_\rho$.

The definition of $T$ given by~\eqref{eq:2} suggests that it can be estimated by 
the data dependent operator
\begin{equation}\label{eq:Tn} 
T_n = \frac{1}{n}\sum_{i=1}^n  K_{x_i}\otimes K_{x_i}.
\end{equation}
The operator $T_n$ is positive and with finite rank; in particular,
$T_n\in\mathcal S_1$ and $\nor{T_n}_{\mathcal S_1} = \tr{T_n} = 1$. We
denote by $(\sigma^{(n)}_j)_{j\in J_n}$ the strictly positive
eigenvalues of $T_n$ (each one repeated according to its multiplicity)
and by {$\big(\sqrt{\sigma^{(n)}_j}\phi^{(n)}_j\big)_{j\in
    J_n}$} the corresponding eigenvectors;
note that in the present case the index set $J_n$ is finite. However,
though $T_n$ converges to $T$ in all relevant topologies (see Lemma
\ref{concentration} and Remark \ref{rem:S1} below), in general
$T^\dag_nT_n$ does not converge to $T^\dag T$ since $T^\dag$ may be
{unbounded}, or, equivalently, since $0$ may be an accumulation point
of the spectrum of $T$ when $\dim\hh = \infty$.  Hence, the problem of
support estimation is {ill-posed}, and regularization techniques are
needed to restore well-posedness and ensure a stable solution. In the
following sections, we will show that spectral regularization
\cite{bapero07,enhane,logerfo08} can be used to learn the support
efficiently from the data.

\subsection{Regularized Estimators via Spectral Filtering}\label{algo}

An approach which is classical in inverse problems (see \cite{enhane}, and also \cite{bapero07,logerfo08} for applications to learning)
consists in replacing the pseudo-inverses $T^\dag_n$ and $T^\dag$ with some bounded approximations 
obtained by {\em filtering out} the components  corresponding to the eigenvalues of $T_n$ and $T$ which are smaller than a fixed regularization parameter $\lambda$. This is achieved by introducing a suitable {\em filter function} $\G:[0,+\infty[\to [0,+\infty[$ and replacing $T^\dag_n$, $T$ with the bounded operators $\G(T_n)$, $\G(T)$ defined by spectral calculus. If the function $\G$ is sufficiently regular, then convergence of $T_n$ to $T$ implies convergence of $\G(T_n)$ to $\G(T)$ in the Hilbert-Schmidt norm. On the other hand, if the regularization parameter $\lambda$ goes to zero, then $\G(T)$ converges to $T^\dag$ in an appropriate sense. We are now going to apply  the same idea to our setting. Since we are interested in approximating the orthogonal projection $P_\rho = T^\dag T = \theta(T)$ rather than the pseudo-inverse $T^\dag$,  we introduce a low-pass filter $\RR$, in a way that the bounded operator $\RR(T)$ is an approximation of $\theta(T)$. In terms of the previously defined function $\G$, this can be achieved by setting $\RR(\sigma){=} \G (\sigma) \sigma$ for all $\sigma\in\R$, so that $\RR(T) = \G (T)T$. Explicitely, in terms of the spectral decompositions of $T_n$ and $T$ we have
\[{
\RR(T_n)=\sum_{j\in J_n} \RR(\sigma_j^{(n)})\  {\big(\sqrt{\sigma_j^{(n)}} \phi^{(n)}_j\big) \otimes \big(\sqrt{\sigma_j^{(n)}} \phi^{(n)}_j\big)}, 
\qquad \RR(T)=\sum_{j\in J} \RR(\sigma_j)\ (\sqrt{\sigma_j} \phi_j)\otimes (\sqrt{\sigma_j} \phi_j) .}
\]
Note that, since the spectra of $T_n$ and $T$ are both contained in the interval $[0,1]$, we can assume that the functions $\G$ and $\RR$ are defined on $[0,1]$. Moreover, as the operators $\RR(T_n)$ and $\RR(T)$ approximate orthogonal projections, it is useful to have the bound $0\leq\RR(T_n),\RR(T)\leq {I}$ satisfied for all $T_n$ and $T$'s, and this can be achieved by choosing the function $\RR$ such that $0\leq\RR(\sigma)\leq 1$ for all $\sigma$.

As a consequence of the above discussion, the characterization of  filter functions giving rise to stable algorithms is captured by the following assumption.

\begin{assm}\label{B}
The family of functions $(\RR)_{\lambda>0}$, with $\RR:[0,1]\to [0,1]$ for all  $\lambda>0$, has the following properties:
\begin{enumerate}[a)]
\item\label{B1} $\RR(0) = 0$ for all $\lambda>0$; 
\item\label{B2} for all $\sigma >0$, we have $\lim_{\la\to 0^+} \RR(\sigma)=1$;
\item\label{B3} for all $\la>0$,  there exists a positive constant $L_\la$ such that 
\[
|\RR(\sigma)-\RR(\tau)|\leq L_\la |\sigma-\tau| \qquad \forall \sigma,\tau \in [0,1] .
\]
\end{enumerate}
\end{assm}
By Assumption \ref{B}.\ref{B1}, there exists a function $\G:[0,1]\to[0,+\infty[$ such that $\RR(\sigma) = \G(\sigma) \sigma$. On the other hand, by Assumption \ref{B}.\ref{B2}) we have $\lim_{\lambda\to 0^+} \RR (\sigma) = \theta (\sigma)$ for all $\sigma\in [0,1]$. Assumption \ref{B}.\ref{B3}) is of technical nature, and will become clear in Section \ref{subsec:cons}; here we note that in particular it implies that $\RR$ is a continuous function for all $\la >0$.

A few examples of filter functions $\RR$ satisfying Assumption \ref{B} and of corresponding functions $\G$ are given in Table \ref{table filter}.
It is easy to check that for each of them  $L_{\la}=1/\la$. See \cite{enhane} for further examples.

\begin{table}[htp]
\begin{center}
  \begin{tabular}[c]{|l|c|c|}
\hline
 & & \\
Tikhonov regularization &
$\displaystyle{\RR(\sigma)=\frac{\sigma}{\sigma+\lambda}}$ & $\displaystyle{\G(\sigma)=\frac{1}{\sigma+\lambda}}$ \\
& &
\\
\hline 
& & \\ 
Spectral cut-off & $\displaystyle{\RR(\sigma)=
  \ind_{]\la,+\infty[}(\sigma)+ \frac{\sigma}{\lambda}
  \ind_{[0,\la]}(\sigma)}$ & $\displaystyle{\G(\sigma)=
  \frac{1}{\sigma}\ind_{]\la,+\infty[}(\sigma)+ \frac{1}{\lambda}
  \ind_{[0,\la]}(\sigma)}$\\
& & \\
\hline 
& & \\ 
Landweber filter & $\displaystyle{\RR(\sigma)=\sigma
     \sum_{k=0}^{{m_\la}}(1-\sigma)^k }$ & $\displaystyle{\G(\sigma)=
     \sum_{k=0}^{{m_\la}}(1-\sigma)^k }$ \\ & & \\
     \hline
\end{tabular}
\end{center}
\caption{Examples of filter functions satisfying Assumption
  \ref{B}. For Landweber filter {$m_\la$ is an integer such
    that $\lim_{\la\to 0} m_\la= \infty$ }.}\label{table filter}
\end{table}

For a chosen filter,   the corresponding regularized empirical estimator of
$F_\rho$ is  defined by 
\begin{equation}\label{estimator}
F_n(x)=\scal{  \RRn (T_n)K_x}{K_x}= {\sum_{j\in J_n} \RRn (\sigma_j^{(n)}) \sigma^{(n)}_j \abs{\phi^{(n)}_j(x)}^2}
 \end{equation}
where we allow the regularization parameter $\la_n$ to depend  on the number of samples  $n$. 
Note that the functions $F_n$ and $F_\rho$ are continuous on $X$ by continuity of the mapping $ x\mapsto K_x$ (see~\ref{P.1.i}) of Proposition \ref{metrica}). In Section~\ref{consistency} we will show that, for an appropriate choice of the sequence $(\la_n)_{n\geq 1}$, 
the estimator $F_n$ converges almost surely to $F_\rho$ uniformly on compact subsets of $X$. Unfortunately, this does not  imply convergence of the $1$-level sets of $F_n$ to the 
$1$-level set of $F_\rho$
in any sense (as, for example,  with respect to the Hausdorff distance).  However,
an estimator of $X_\rho$ can be obtained by setting
\begin{equation}\label{set_est}
X_n=\{x\in X\mid F_n(x)\geq 1-\tau_n\},
\end{equation}
where $\tau_n > 0$  is an off-set parameter that depends on the sample size $n$ (recall that $F_n$ takes values in $[0,1]$).
In  Section~\ref{consistency} we show that, for   a suitable choice of the sequence $(\tau_n)_{n\geq 1}$, the {closed} set $X_n$ is indeed a 
consistent estimator of the support with respect to the Hausdorff distance.

In the following section 
we discuss some remarks about the  computation of $F_n$.

\subsection{Algorithmic and Computational Aspects}\label{complexity}

We show that the computation of $F_n$ (hence of   $X_n$) reduces to a  finite 
dimensional problem involving the empirical kernel matrix defined by the   data. 
To this purpose, it is useful to introduce the sampling operator  
\begin{equation}\label{eq:def_Sn}
S_n:\hh\to {\R^n}\qquad  S_n f = \col{f(x_1) \\ \vdots \\ f(x_n)},
\end{equation}
which can be  interpreted as the restriction operator which evaluates functions in $\hh$
on the points of the training set.
The {transpose} of $S_n$ is 
\[
S_n^{\top}: {\R^n} \to \hh\qquad  S^{\top}_n \col{\alpha_1 \\ \vdots \\ \alpha_n} =\sum_{i=1}^n \alpha_i K_{x_i}, 
\]
and $S_n^{\top}$ can  be interpreted as the out-of-sample extension operator \cite{coilaf06,RoBeDe10}. 
A simple computation shows that 
\[
T_n=\frac{1}{n} S_n^{\top} S_n\qquad S_nS_n^{\top}={\mathbf K}_n\qquad \bigl({\mathbf K}_n\bigr)_{ij}= K(x_i,x_j).\]
Hence, considering the filter given in the form $\RR(T_n)=\G(T_n)T_n$, we have
\[ \RR(T_n)= \G \left(\frac{S_n^{\top} S_n}{n}\right)\frac{S_n^{\top}
  S_n}{n}=\frac{1}{n} S_n^{\top}\, \G\left(\frac{S_nS^{\top}_n}{n}\right)\,S_n
=\frac{1}{n} S_n^{\top}\, \G\left(\frac{{\mathbf K}_n}{n}\right)\,S_n,
\]
where the second equality follows from  spectral
calculus. Using the definition of the sampling operator, we   can
consider  the $n$-dimensional vector ${\mathbf K}_x$ defined by
\[
{\mathbf K}_x {=} S_n K_x=\col{K(x_1,x) \\ \vdots \\ K(x_n,x)},
\]
{and}  \eqref{estimator} can be written as
\begin{equation}\label{eq:estimator2}
F_n(x)  = \scal{  \RRn (T_n)K_x}{K_x}= \scal{\frac{1}{n} \, \G\left(\frac{{\mathbf K}_n}{n}\right)\,S_nK_x}{S_n K_x} 
= \frac{1}{n} \, {\mathbf K}_x^\ast\, \Gn\left(\frac{{\mathbf
      K}_n}{n}\right)\,{\mathbf K}_x ,   
\end{equation}
where ${\mathbf K}_x^\ast$ is the conjugate transpose of ${\mathbf K}_x$.
More explicitly we have 
\begin{equation}\label{rep}
F_n(x)  = \sum_{i=1}^n \alpha_i(x) K(x,x_i)   \qquad  \alpha_i(x) = \frac{1}{n}\sum_{j=1}^n \left(g_{\lambda_n}\left (\frac{\mathbf K_n}{n}\right)\right)_{ij} K(x_{j},x).
\end{equation}
The above equation shows that, while  $\hh$ could be infinite
  dimensional, the computation of the estimator reduces to a finite
  dimensional problem. Further, 
though the mathematical definition of the filter is done 
through spectral calculus,   the computations might not require 
performing an eigen-decomposition.
As an example,  for Tikhonov regularization {
  $g_{\la_n}(\sigma)=\frac{1}{\sigma+\la_n}$, so that 
$g_{\lambda_n}\left (\frac{\mathbf K_n}{n}\right)= (\frac{{\mathbf
    K}_n}{n}+\la_n)^{-1}$ and }
the coefficient vector  $\alpha(x)$ in   \eqref{rep} is  given by 
$$
\alpha(x)=({\mathbf K}_n+n\la_n)^{-1}{\mathbf K}_x.
$$ 
In the case of the Landweber filter, it is possible to prove that 
the coefficient vector can be evaluated iteratively by setting $\alpha^0(x)=0$, and 
$$
\alpha^{t}(x)=\alpha^{t-1}(x) +\frac 1 n ({\mathbf K}_x-{\mathbf K}_n\alpha^{t-1}(x))
$$
for {$t=1, \dots, m_{\la_n}$.} {We refer to \cite{logerfo08} for the corresponding
  algorithm in a supervised framework; see also the discussion in Section~\ref{IPERM}}.

We thus see that the estimator corresponding to Tikhonov
regularization can be computed via Cholesky decomposition and has
complexity of order $O(n^3)$.  For Landweber iteration the complexity
is $O(n^2m)$, where $m$ is the number of iterations. Finally, the
spectral cut-off, or truncated SVD, requires $O(n^3)$ operations to
compute the eigen-decompostion of the kernel matrix. Further
discussions can be found in \cite{logerfo08} and references therein.
We end remarking that, in order to test whether $N$ points belong or
not to the support, we simply have to repeat the above computation
replacing ${\mathbf K}_x$ by a $n\times N$ matrix ${\mathbf K}_{x,N}$,
in which each column is a vector ${\mathbf K}_x$ corresponding to a
point $x$ in the test set. Note that in this case the coefficients
$\alpha(x)$ will also form a $n\times N$ matrix.

\section{Error Analysis: Convergence and Stability}\label{consistency}

In this section we develop an errror analysis for the proposed class
of estimators.  First, we discuss convergence (consistency) and then
stability with respect to random sampling in terms of finite sample
bounds. We continue to suppose throughout this section that Assumption
\ref{A} holds true, and consider $X$ as a metric space with metric $\dk$.
 
\subsection{Empirical data}\label{sec:misurismi}
We recall that the empirical data are a set of i.i.d.~points $x_1, \dots,
x_n$, each one drawn from $X$ with probability
$\rho$.  Since we need to study  asymptotic properties when the sample size $n$
goes to infinity,  we introduce the following probability space
\begin{equation}
\Omega=\set{(x_i)_{i\geq 1}\mid x_i\in X\ \forall i\geq 1} ,\label{eq:5}
\end{equation}
endowed with the product $\sigma$-algebra $\A{\Omega} {=} \A{X} \otimes \A{X} \otimes \ldots $ and the product probability measure $\mathbb P {=} \rho\otimes\rho\otimes\ldots$.  We recall that, given {an integer $n$ and a topological space $M$ endowed with the $\sigma$-algebra of its Borel subsets}, an  {\em $M$-valued
    estimator} of size $n$  is a measurable map $\Xi_n:\Omega\to M$  depending only on the first
  $n$-variables, that is
\[
\Xi_n(\omega)=\xi_n(x_1,\ldots,x_n)\qquad
\omega =(x_i)_{i\geq 1}
\]
for some measurable map $\xi_n : X^n \to M$.
The number $n$ is the cardinality of the sampled data. We then have the following facts.
\begin{prop}\label{prop:misurabilita}
For all $n\geq 1$
\begin{enumerate}[i)]
\item $T_n$ is a $\mathcal S_k$-valued estimator for $k=1,2$;
\item if $X$ is locally compact, then $F_n$ is a $C(X)$-valued estimator, where $C(X)$ is the space of continuous functions on $X$ with the topology of uniform convergence on compact subsets.
\end{enumerate}
\end{prop}
The proof of the above proposition is rather technical, and we defer
the interested reader to~\ref{maurer} for more details. 
\begin{rem}
{In item ii) of Proposition \ref{prop:misurabilita}, the assumption that $X$ is locally compact is needed to ensure that
the topology of uniform convergence on compact subsets is a {separable}
metric topology on $C(X)$, which in turn is essential to prove measurability of the random variable $F_n$ (see the proof of Proposition \ref{prop:misurabilita2} in \ref{maurer}). In many examples, the set $X$ has its own locally compact separable metric $d_X$. In this case, in order for $X$ to be  locally compact metric space also for the metric $d_K$, it is enough that the kernel $K$ is a $d_X$-continuous function separating every
subset of $X$ which is closed with respect to $d_X$, as the two
topologies induced by $d_X$ and $d_K$ then coincide by item ii) of
Proposition \ref{c:connection}.

If $X$ is not locally compact  (which we will regard as a pathological case), then, in order to have
measurability of $F_n$, one needs to replace the probability measure
$\mathbb P$ with the outer measure (see the discussion in Section 2 of 
\cite{kolt11} and in Section 1.7 of \cite{vawe96}).}
\end{rem}

\begin{rem}
Statisticians adopt a different notation: the data are described by a family $Y_1,Y_2,\ldots$ of random variables taking value in $X$,
  each defined on the same probability space $(\Gamma,\A{\Gamma},\mathbb Q)$, which are {i.i.d.~according to} $\rho$. An $M$-valued estimator of
  size $n$ is then simply a random variable $\xi_n (Y_1 , \ldots ,
  Y_n)$, where $\xi_n:X^n\to M$ is a measurable map. The equivalence
  between the two approaches is made clear by setting
  $(\Gamma,\A{\Gamma},\mathbb Q) \equiv (\Omega,\A{\Omega},{\mathbb
  P})$ and $Y_i(\omega) = x_i$ for all $\omega = (x_j)_{j\geq 1}$ and
  $i\geq 1$.
\end{rem}
Concentration of measure results  for  random variables in Hilbert spaces can be used 
to prove that $T_n$ is an unbiased  estimator of $T$, as stated in the
following lemma. 
\begin{lem}\label{concentration}
For $n\geq 1$ and $\delta > 0$,  
\begin{equation}\label{eq:concentration_0}
\nor{T-T_n}_{\mathcal S_2}\leq \frac{2(\delta \vee \sqrt{2\delta})}{\sqrt{n}}
\end{equation}
with probability at least $1-2e^{-\delta}$. Furthermore 
\begin{equation}\label{eq:concentration}
\lim_{n\to \infty} \frac{\sqrt{n}}{\log n} \nor{T-T_n}_{\mathcal S_2}=0\qquad\text{almost surely}.
\end{equation}
\end{lem}
\begin{proof}
The result is known, but we report its short proof.
For all $i\geq 1$ define   the random variables $Z_i:\Omega\to \mathcal S_2 $ as 
\[
Z_i(\omega)=K_{x_i}\otimes K_{x_i}\qquad \omega=(x_j)_{j\geq 1}\in\Omega.\]
The fact that $Z_i$ is measurable follows
from Lemma \ref{lem:app1} in~\ref{maurer}. 
Then, for all $i\geq 1$, we have $\nor{Z_i}_{\mathcal S_2}\leq 1$ almost surely, 
$\EE{Z_i}=T$, and clearly $\EE{\nor{Z_i}^2_{\mathcal S_2}}\leq 1$.
The first result follows easily applying Lemma \ref{conc_lemma} in~\ref{sec:conc} and simplifying the right hand side of \eqref{bernie}, and the second is a consequence of Lemma \ref{cor_conc} in~\ref{sec:conc}.
\end{proof}

\begin{rem}\label{rem:S1}
Note that \eqref{eq:concentration} and Theorem 2.19 in \cite{Simon} imply that
$$
\lim_{n\to \infty} \nor{T-T_n}_{\mathcal S_1}=0 \qquad \text{almost surely}.
$$
\end{rem}

\subsection{Consistency}\label{subsec:cons}

We now choose a family of filter functions $(\RR)_{\lambda>0}$ and study the convergence of the associated estimators $F_n$ and $X_n$ introduced in Section \ref{sec:learn}.

{\subsubsection{Consistency of $F_n$}}

We begin proving convergence of the functions $F_n$ defined in \eqref{estimator} to the function $F_\rho$ in \eqref{Projection}. We introduce the map $G_{\la}:X\to\R$ defined by 
\begin{equation*}\label{intermediate}
G_{\la}(x)=\scal{ \RR(T)K_x}{K_x} \qquad \forall x\in X,
\end{equation*}
which can be seen as the {\em infinite sample} analogue of $F_n$. Clearly, $G_{\la}$ is a continuous function. For all sets $C\subset X$, we then have the following splitting of the error into two parts, the  {\em sample error} and the  {\em approximation error}
\begin{equation}\label{decomp}
\sup_{x\in C}  \abs{F_n(x)-F_\rho(x)} \leq \underbrace{ \sup_{x\in C}
\abs{F_n(x)-G_{\la_n}(x)}}_{\text{sample error}} + \underbrace{ \sup_{x\in C}
\abs{G_{\la_n}(x)-F_\rho(x)}}_{\text{approximation error}}.  
\end{equation}
In order to prove consistency, we need to show that the left hand side goes to $0$ as the sequence of regularization parameters $(\la_n)_{n\geq 1}$ tends to $0$. This will be done separately for the approximation and the sample errors in the next two propositions.

\begin{prop}\label{appr_err}
Under Assumption \ref{B}.\ref{B2}), if the sequence $(\la_n)_{n\geq 1}$ is such that
$ \lim_{n\to \infty}\lambda_n=0$, then, for any  compact subset $C\subset X$, 
\begin{equation*}\label{apprx_conv}
  \lim_{n\to \infty} \sup_{x\in C} \lvert G_{\la_n}(x)-F_\rho(x) \rvert=0.
  \end{equation*}
\end{prop}
\begin{proof}
Assumption \ref{B}.\ref{B2}) and  $\lim_{n\to \infty}\la_n=0$  imply
that the sequence of non-negative functions $(\RRn)_{n\geq 1}$ is bounded by $1$ and converges pointwisely to the Heaviside function
$\theta$ on the interval $[0,1]$. Spectral theorem ensures that, for all $x\in C$, 
\begin{equation}
 \lim_{n\to\infty} \RRn (T)K_x=\theta(T)K_x.\label{dubbio}
\end{equation}
Given $\eps>0$, by compactness of $C$ there exists a finite covering of
$C$ by balls of radius $\eps$, namely $C \subset \cup_{i=1}^m B(x_i,\eps)$.
By~\eqref{dubbio} there exists 
$n_0$ such that
\[ \max_{i\in\{1,\ldots,m\}}   \nor{r_{\la_n}(T)K_{x_i}-\theta(T)K_{x_i}}\leq \eps \qquad \forall n\geq n_0.\]
Hence, for all $n\geq n_0$, we have
\begin{align*}
   \sup_{x\in C} \abs{G_{\la_n}(x) - F_\rho(x)} & = \sup_{x\in C} \abs{\scal{(r_{\la_n}(T)-\theta(T))K_x }{K_x}} \\
 & \leq \sup_{x\in C} \nor{K_x} \, \sup_{x\in C}  \nor{(r_{\la_n}(T)-\theta(T))K_x } \\
 & \leq \max_{i\in\{1,\ldots,m\}} \sup_{x\in B(x_i,\eps)} \nor{(r_{\la_n}(T)-\theta(T))K_{x_i} + (r_{\la_n}(T)-\theta(T))(K_x - K_{x_i}) } \\
  & \leq \max_{i\in\{1,\ldots,m\}} \sup_{x\in B(x_i,\eps)}  \bigl(\nor{
    (r_{\la_n}(T)-\theta(T))K_{x_i}} +
  \nor{r_{\la_n}(T)-\theta(T)}_\infty \nor{K_x-K_{x_i}}\bigr) \\
& \leq \eps + \eps \sup_{\sigma\in[0,1]} \abs{r_{\la_n}(\sigma)-\theta(\sigma)}
= 3\eps,
\end{align*}
where $\nor{K_x-K_{x_i}}<\eps$ for all $x\in B(x_i,\eps)$ since
$\nor{K_x-K_{x_i}} = \dk (x,x_i)$, an, because $\abs{r_{\la_n}(\sigma)} \leq 1$, $\abs{\theta(\sigma)} \leq 1$,   $\sup_{\sigma\in[0,1]}
\abs{r_{\la_n}(\sigma)-\theta(\sigma)}\leq 2$.
\end{proof}

Convergence to zero of the sample error follows from \eqref{eq:concentration} and the next proposition.

\begin{prop}\label{finite_bound}
For all sets $C\subset X$ we have
\begin{equation}\label{decomp2}
\sup_{x\in C} \abs{F_n(x)-G_{\la_n}(x)} \leq \nor{r_{\la_n}(T_n)-r_{\la_n}(T)}_{\mathcal S_2} .
\end{equation}
In particular, if Assumption \ref{B}.\ref{B3}) holds, then
\begin{equation}\label{eq:inpiu}
\sup_{x\in C} \abs{F_n(x)-G_{\la_n}(x)} \leq L_{\la_n}\nor{T_n-T}_{{\cal S}_2} .
\end{equation}
\end{prop}
\begin{proof}
For all $x\in X$, we have the bound
\begin{align*}
 \abs{F_n(x)-G_{\la_n}(x)} & = 
 \abs{\scal{(\RRn (T_n)-\RRn (T))K_x}{K_x}} \\
 & \leq \nor{r_{\la_n}(T_n)-r_{\la_n}(T)}_{\infty} \, \nor{K_x}^2 \\
 & \leq \nor{r_{\la_n}(T_n)-r_{\la_n}(T)}_{\mathcal S_2} ,
 \end{align*}
which proves \eqref{decomp2}.
Assumption~\ref{B}.\ref{B3}) and  Theorem 8.1  in  \cite{biso03} (see also  
 Lemma~\ref{lem:ineMaurer} in~\ref{maurer2} for a simple unpublished proof due to A.~Maurer) imply that
\[
\nor{ r_{\la_n}(T_n)-r_{\la_n}(T)}_{\mathcal S_2}\leq L_{\la_n}\nor{T_n-T}_{\mathcal S_2} .
\]
Inequality \eqref{eq:inpiu} then follows.
\end{proof}

The above results can be combined in the following theorem, showing that, if the sequence $\la_n$ is suitably chosen, then  $F_n$ converges almost surely to $F_\rho$ with respect to the topology of uniform convergence on compact subsets of $X$. 
 
\begin{thm}\label{thm:consistency}
Under Assumption \ref{B}, if the sequence $(\la_n)_{n\geq 1}$ is such that
\begin{equation}\label{ParChoice} 
 \lim_{n\to\infty}\lambda_n=0 \quad \text{and}\quad 
\sup_{n\geq 1} \frac {L_{\la_n} \log n}{\sqrt{n}}<+\infty,
 \end{equation}
  then,  for every  compact subset $C\subset X$, 
\begin{equation}\label{Consistency}
  \lim_{n\to\infty} \sup_{x\in C} \lvert F_n(x)-F_\rho(x)\rvert=0
  \qquad\text{almost surely}.
  \end{equation}
\end{thm}
\begin{proof}
We show convergence to zero of both the two terms in the right hand side of inequality \eqref{decomp}, thus implying \eqref{Consistency}. By \eqref{eq:inpiu}, we have
$$
\sup_{x\in C} \abs{F_n(x)-G_{\la_n}(x)} \leq L_{\la_n} \nor{T_n-T}_{\mathcal S_2} = \frac {L_{\la_n} \log n}{\sqrt{n}}\frac{\sqrt{n}\nor{T_n-T}_{\mathcal S_2}}{\log n } \leq M \frac{\sqrt{n}\nor{T_n-T}_{\mathcal S_2}}{\log n } , 
$$
where $M=\sup_{n\geq 1} (L_{\la_n} \log n)/\sqrt{n}$ is finite
by~\eqref{ParChoice}. 
Then \eqref{eq:concentration} 
implies that the first term in the right hand side of inequality \eqref{decomp} converges to zero almost surely. Since the second term goes to zero by Proposition \ref{appr_err}, the claim follows.
\end{proof}

{\subsubsection{Consistency of $X_n$}}

{ As already remarked above, uniform convergence of $F_n$ to
$F_\rho$ on compact subsets {\em does not}  imply convergence of
the level sets of $F_n$ to the corresponding level sets of $F_\rho$ in any sense (as, for example,  with respect to the Hausdorff distance among
compact subsets). For this reason,  we  introduce a family of threshold parameters $(\tau_n)_{n\geq 1}$ and define the estimator $X_n$ of the set $X_\rho$ as in \eqref{set_est}.

We define a data dependent parameter $\tau_n$ as the function on $\Omega$
\begin{equation}\label{eq:deftaun}
\tau_n(\omega)= 1-\min_{1\leq i \leq n}[F_n(\omega)](x_i) \qquad \omega=(x_i)_{i\geq 1},
\end{equation}
where we wrote explicitely the dependence of $F_n$ on the training set $\omega\in\Omega$. Since $F_n$ takes values in $[0,1]$, clearly $\tau_n(\omega)\in [0,1]$.

\begin{prop}\label{prop:limtaun}
Suppose the metric space $X$ is compact. Then, under Assumption \ref{B}, the function $\tau_n$ is a $\R$-valued estimator. Moreover, if the sequence  $(\la_n)_{n\geq 1}$ satisfies~\eqref{ParChoice}, we have
$$
\lim_{n\to\infty} \tau_n = 0 \qquad \text{almost surely} .
$$
\end{prop}
\begin{proof}
The proof that $\tau_n$ is a $\R$-valued estimator is of technical nature, and we postpone it to Proposition \ref{prop:misurtaun} in \ref{maurer}.\\
Here we prove that $\lim_{n\to\infty} \tau_n = 0$ with probabiliy $1$. By Theorem \ref{thm:consistency}, we can find an event $E_1\subset\Omega$ with $\PP{E_1}=1$ such that $\lim_{n\to\infty} \sup_{x\in X} \lvert F_n(x)-F_\rho(x)\rvert=0$ on $E_1$. Moreover, for the event $E_2=\{x_i \in X_\rho \mbox{ for all } i\geq 1\}$, we clearly have $\PP{E_2}=1$ by definition of $X_\rho$ and $\mathbb{P}$. If $\omega\in
E_1\cap E_2$ and $\eps>0$ is fixed, then there exists $n_0 \geq 1$ (possibly
depending on $\omega$ and $\eps$) such that for all $n\geq n_0$ 
$|[F_n(\omega)](x)-F_\rho(x)|\leq \eps$ for all $x\in X$.  Since $F_\rho(x)=1$ for all $x\in X_\rho$ by definition and $x_1,\ldots,x_n\in X_\rho$, it
follows that  $|[F_n(\omega)](x_i)-1| \leq \eps$ for all $1\leq i\leq n$, that is
\[
0\leq 1-[F_n(\omega)](x_i)  \leq \eps \qquad \forall i\in\{1,2,\ldots,n\},
\]
so that $0\leq \tau_n(\omega)  \leq \eps$. Thus, $\lim_{n\to\infty} \tau_n(\omega) = 0$, and,
since $\PP{E_1\cap E_2}=1$, the  sequence $(\tau_n)_{n\geq 1}$ goes to zero with probability $1$.
\end{proof}

The following is the central result of this section. It shows that, assuming $X$ is compact and  for the above choice of the sequence $(\tau_n)_{n\geq 1}$, the Hausdorff distance between $X_n$ and $X_\rho$ goes to zero with probability $1$. Here we recall that  the Hausdorff distance between two
subsets  $A,B\subset X$ is
\[
d_H (A,B)=\max\set{\sup_{a\in A}\dk(a,B),\ \sup_{b\in B}\dk(b,A)},
\]
where $\dk (x,Y) = \inf_{y\in Y} \dk (x,y)$.

\begin{thm}\label{thm:hauss}
Suppose the metric space $X$ is compact. Under Assumption \ref{B}, if $\hh$ separates the set $X_\rho$ and the sequence  $(\la_n)_{n\geq 1}$ satisfies~\eqref{ParChoice}, for the choice of the threshold parameters $(\tau_n)_{n\geq 1}$ given in \eqref{eq:deftaun} we have
\[
\lim_{n\to \infty} d_H (X_n, X_\rho)=0 \qquad \text{almost surely} .
\]
\end{thm}

We devote the rest of this section to proof of the above theorem. For simplicity, we split it into a few lemmas.

\begin{lem}\label{lem:agg1}
Under the hypotheses of Theorem \ref{thm:hauss}, we have
\begin{equation}\label{eq:limXrho}
\lim_{n\to\infty} \sup_{x\in X_n}d_K(x,X_\rho) =0 \qquad \mbox{almost surely}.
\end{equation}
\end{lem}
\begin{proof}
Let $E$ be the event $E = \{\lim_{n\to\infty}\tau_n = 0\}$. Then, $\PP{E} = 1$ by Proposition \ref{prop:limtaun}. We fix $\omega\in E$, and suppose by contradiction that at such $\omega$ the limit \eqref{eq:limXrho} does not hold. Then (depending on $\omega$) there exists $\eps>0$ such
that for all $k$ there is $n_k\geq k$ satisfying the inequality $\sup_{x\in X_{n_k}}\dk(x,X_\rho)\geq 2\eps$. Hence there is $z_k\in
X_{n_k}$ such that 
\begin{equation}
\dk(z_k, x)\geq \eps\qquad \text{for all }x\in X_\rho.\label{eq:1}
\end{equation}
Since $X$ is compact, possibly passing to a
subsequence we can assume that the sequence $(z_k)_{k\geq 1}$ converges to a limit
$z\in X$. We claim that $z\in X_\rho$. Indeed, if $k$ is sufficiently large, then we have
\begin{align*}
\abs{F_\rho(z) -1} & \leq \abs{F_\rho(z)-F_\rho(z_k)} + \abs{F_\rho(z_k) - F_{n_k}(z_k)} + \abs{F_{n_k}(z_k) -1} \\
&  \leq \abs{F_\rho(z)-F_\rho(z_k)}+ \sup_{x\in X} \abs{F_\rho(x) - F_{n_k}(x)} + \tau_{n_k} , \\
\end{align*}
where
$\abs{F_{n_k}(z_k) -1} \leq \tau_{n_k}$ is due to the fact that $z_k\in X_{n_k}$, so that
\[ 1+\tau_{n_k}\geq 1 \geq F_{n_k}(z_k)\geq 1 -\tau_{n_k}.\]
As $n_k$ goes to $\infty$, we have $\sup_{x\in X} \abs{F_\rho(x) - F_{n_k}(x)} \to 0$ by Theorem \ref{thm:consistency}; moreover, since $F_\rho$ is continuous in $z$
and $\tau_{n_k}$ goes to zero, the above inequality for $\abs{F_\rho(z)-1}$ gives $F_\rho(z)=1$. Since $\hh$ separates $X_\rho$, this implies $z\in X_\rho$. However, \eqref{eq:1} implies that $\dk(z,x)\geq \eps$ 
for all $x\in X_\rho$, which is the desired contradiction.
\end{proof}

The proof that  $\sup_{x\in X_n}d_K(x,X_\rho)$ goes to zero as $n\to\infty$ requires a further technical lemma, see  \cite[Lemma 6.1]{GY02}. In its statement, for all $n\geq 1$ and $x\in X$, we denote by $\xi_{1,n}(x)$ the nearest neighbour of $x$ in the training set $\set{x_1,\ldots,x_n}$, i.e.
$$
\xi_{1,n}(x) = {\rm arg\, min}_{x_1,x_2,\ldots x_n} d_K (x_i,x) .
$$
\begin{lem}\label{lem:agg2}
For all $x\in X_\rho$,
\[
\lim_{n\to\infty} d_K(\xi_{1,n}(x),x)=0 \qquad \mbox{almost surely}.
\]
\end{lem}
\begin{proof}
Given $x\in X_\rho$, fix $\eps>0$ and, denoted by $B(x,\eps)$ the
closed ball with center $x$ and radius $\eps$, set
 $p=\rho(B(x,\eps))$. By definition of the support and the fact that
 $\rho$ is a probability measure, $0<p\leq 1$.  Furthermore
 \begin{align*}
   \PP{ d_K(\xi_{1,n}(x),x)>\eps} & =\PP{x_i\not\in B(x,\eps)\, \forall
     i=1,\ldots,n} \\ 
\mbox{(by independence of the $x_i$'s)}& = \Pi_{i=1}^n \PP{x_i\not\in B(x,\eps)} \\
\mbox{(since the $x_i$'s are identically distributed)} & = \Pi_{i=1}^n (1-\rho(B(x,\eps))\\
& = (1-p)^n .
\end{align*}
Since $0\leq 1-p<1$, the series $\sum_n(1-p)^n$ converges, so that Borel-Cantelli lemma yields
\[
\PP{\bigcup_{n=1}^\infty \bigcap_{m=n}^\infty\set{ d_K(\xi_{1,m}(x),x) \leq \eps}}= 1 .
\]
Since this holds for all $\eps>0$, we have
\[
\PP{\bigcap_{k=1}^\infty \bigcup_{n=1}^\infty \bigcap_{m=n}^\infty\set{ d_K(\xi_{1,m}(x),x) \leq \frac{1}{k}}}= 1 ,
\]
and the lemma follows.
\end{proof}

\begin{lem}\label{lem:agg3}
Under the hypotheses of Theorem \ref{thm:consistency}, if the metric space $X$ is compact, then
\begin{equation}\label{eq:limXn}
\lim_{n\to \infty} \sup_{x\in X_\rho}\dk(x,X_n)  =0 \qquad \mbox{almost surely} .
\end{equation}
\end{lem}
\begin{proof}
Choose a denumerable dense family $\set{z_j}_{j\in J}$ in $X_\rho$. By the Lemma \ref{lem:agg2} there exists an event $E$ with
probability $1$ such that
\begin{equation}\label{eq:altrolim}
\lim_{n\to+\infty} d_K(\xi_{1,n}(z_j),z_j)=0 \qquad \forall j\in J
\end{equation}
on $E$. We claim that the limit \eqref{eq:limXn} holds on $E$. Observe that, by definition of $\tau_n$,  $x_i\in
X_n$ for all $1\leq i \leq n$, and 
\[
\sup_{x\in X_\rho}d_K(x,X_n)\leq \sup_{x\in X_\rho}\min_{1\leq i\leq n} d_K(x,x_i) =
\sup_{x\in  X_\rho}  d_K(\xi_{1,n}(x),x),
\]
so that it is enough to show that  $ \lim_{n\to+\infty}
\sup_{x\in  X_\rho}  d_K(\xi_{1,n}(x),x)=0$.\\
Fix $\eps>0$. Since $X_\rho$ is compact, there is a finite subset
$J_\eps\subset J$ such that $\set{B(z_j,\eps)}_{j\in J_\eps}$ is a finite covering of
$X_\rho$. We claim that  
\begin{equation}
\sup_{x\in  X_\rho}  d_K(\xi_{1,n}(x),x)\leq \max_{j\in J_\eps} d_K(\xi_{1,n}(z_j),z_j)+\eps.\label{goin}
\end{equation}
Indeed, fixed $x\in X_\rho$, there exists an index $j\in J_\eps$ such that
$x\in B(z_j,\eps)$.  By definition of $\xi_{1,n}$, clearly
\[ d_K(\xi_{1,n}(x),x)\leq  d_K(\xi_{1,n}(z_j),x),\]
so that by the triangular inequality we get 
\begin{align*}
  d_K(\xi_{1,n}(x),x)& \leq d_K(\xi_{1,n}(z_j),x)\leq
  d_K(\xi_{1,n}(z_j),z_j)+ d_K(z_j,x) \\
& \leq
  d_K(\xi_{1,n}(z_j),z_j)+\eps \\
& \leq \max_{j\in J_\eps}
  d_K(\xi_{1,n}(z_j),z_j)+\eps.
\end{align*}
Taking the $\sup$ over $X_\rho$ we get the claim.\\
Since $J_\eps$ is finite, by \eqref{eq:altrolim}
\[
\lim_{n\to+\infty} \max_{j\in J_\eps} d_K(\xi_{1,n}(z_j),z_j)=0,
\]
hence \eqref{goin} yelds
\[ 
\limsup_{n\to \infty} \sup_{x\in  X_\rho}
d_K(\xi_{1,n}(x),x) \leq \eps.
\]
Since $\eps$ is arbitrary, we get 
$\lim_{n\to+\infty} \sup_{x\in  X_\rho}
d_K(\xi_{1,n}(x),x)=0$, and this concludes the proof
\end{proof}

The proof of Theorem \ref{thm:hauss} follows easily combining the previous lemmas.

\begin{proof}[Proof of Theorem \ref{thm:hauss}]
As $d_H (X_n,X_\rho)=\max\{\sup_{x\in X_n}\dk(x,X_\rho),\ \sup_{x\in X_\rho}\dk(x,X_n)\}$, the theorem follows combining Lemmas \ref{lem:agg1} and \ref{lem:agg3}.
\end{proof}

We conclude this section with some comments.
First, if $\hh$ does
not separate $X_\rho$, then the statement of Theorem \ref{thm:hauss} continues to be
true provided that the support  $X_\rho$ is replaced by the level set $\set{x\in X\mid F_\rho(x)=1}$.
Note that,  the Hausdorff distance $d_H$
has been defined with respect to the metric $\dk$ induced by the kernel, however, if the set $X$ has its own metric $d_X$ making it compact and the hypotheses of Proposition \ref{c:connection} are satisfied, then Theorem \ref{thm:hauss} implies convergence of $X_n$ to $X_\rho$ also with respect to the Hausdorff distance associated to $d_X$.
Finally, we remark that in Theorem \ref{thm:hauss} convergence of $X_n$ to $X_\rho$ does not depend on any a priori assumption on the probability $\rho$.}

 \subsection{Finite Sample Bounds and Stability of Random Sampling}

In order to prove stability of our algorithms under random sampling and determine their convergence rates, we need
 to specify suitable a priori assumptions on the class of problems to
 be considered. In the present section, a detailed analysis {of the convergence rates of $F_n$ to $F_\rho$} will be carried out for the case of the 
 Tikhonov filter $\RR(\sigma) = \sigma/(\sigma+\la)$. The techniques in \cite{cap06} should allow to 
 derive similar results for filters other than Tikhonov.
 
For all $\lambda>0$ we define
\[ {\mathcal N}(\la)=\tr{(T+\la)^{-1}T}=\sum_{j\in J} \frac{\sigma_j}{\sigma_j+\la},\]
which is finite since $T$ is a trace class operator. 
The above quantity is related to the degrees of freedom of the estimator \cite{hatifr01}.
Here, we recall that ${\mathcal N}$ is a decreasing function of $\la$ and $\lim_{\la\to
  0^+}{\mathcal N}(\la)=N$, where $N$ is the dimension of  the range of $T$.

The a priori conditions we consider in the present paper are given  by the following two assumptions, which 
involve both the reproducing kernel $K$ and the probability measure $\rho$ (compare with {\cite{cap07,cade07}}).

\begin{assm}\label{C}
We assume that 
\begin{enumerate}[a)]
\item\label{C2} there exist $b\in [0,1]$ and $D_b\geq 1$ such that
\begin{equation}\label{prior2}
\sup_{\la>0} {\mathcal N}(\la)\la^b\leq D_b^2 ;
\end{equation} 
\item\label{C1}  there exist $ 0< s\leq 1$ and a constant $C_s>0$ such that $P_\rho K_x\in
  \ran{T^{s/2}}$ for all $x\in X$, and 
\begin{equation}\label{prior1}
\sup_{x\in X}\nor{T^{-\frac{s}{2}} P_\rho K_x}^2\leq C_s .
\end{equation}
\end{enumerate}
\end{assm}

The above conditions are classical in the theory of inverse problems and have been 
recently considered in supervised learning. Before showing how they allow to derive a finite sample bound on the error $\sup_{x\in X}  \abs{F_n(x)-F_\rho(x)}$, we add some comments. First, Assumption \ref{C}.\ref{C2}) is related to the level of ill-posedness of the problem \cite{enhane}
and can be interpreted as a condition specifying the {\em aspect ratio} of the
range of $T$.  Since $0<\la{\mathcal N}(\la) < \tr{T} = 1$, inequality~\eqref{prior2} is always satisfied with the choice $b=1$ and $D_1 = 1$, so that in this case we are not imposing any a priori assumption. If $\dim\ran{T} = N <\infty$, the best choice is $b=0$ and $D_0 = \sqrt{N}$; otherwise, if $\dim\ran{T} = \infty$, then necessarily $b>0$. In the latter case, a sufficient condition to have $b<1$ is to assume  a  decay rate  $\sigma_j\sim j^{-1/b}$ on the eigenvalues
of $T$   (see Proposition~3 of \cite{cade07}).

Coming to Assumption \ref{C}.\ref{C1}), first of all we remark that it
is always satisfied when $\dim\ran{T}$ is finite with the choice $s=1$
and $C_1 = \max_{j\in J} 1/\sigma_j$. In the general case,
{Assumption \ref{C}.\ref{C1})} can be expressed {by  the following
equivalent condition}
\begin{equation}\label{apriori}
{\sum_{j\in J} \sigma_j^{1-s} |\phi_j(x)|^2  \leq C_s  \qquad \forall x\in X ,}
\end{equation}
where {$(\phi_j,\sigma_j)_{j\in J}$ are the eigenvectors and
  eigenvalues of $L_K$,}  which {were} defined in
Section \ref{sec:int} (see in particular \eqref{eq:outofsample} for
the definition of the functions $\phi_j$ outside the set $X_\rho$).
Clearly, the higher is $s$, the stronger is the assumption.

Note that in particular inequality~\eqref{apriori} holds true if there exists a constant\footnote{As as it happens for example for reproducing kernels on $X=[0,2\pi]^d$ which are invariant under translations, when $\rho$ is the Lebesgue measure on  $[0,2\pi]^d$.} $\kappa>0$ such that $\sup_{x\in X}\abs{\phi_j(x)}\leq \kappa$ for all $j\in J$,  and $s\in ]0,1]$ is chosen to make the series 
$\sum_{j\in J} \sigma_j^{1-s}$ finite. In this case, it is quite easy to give conditions on the eigenvalues $(\sigma_j)_{j\in J}$ assuring that both Assumptions \ref{C}.\ref{C2}) and \ref{C}.\ref{C1}) are satisfied. For example, if  $\sigma_j\sim j^{-1/b}$ for
some $0<b<1$, then \eqref{prior2} holds true with this choice of $b$, 
and~\eqref{prior1} is satisfied for any $0<s<1-b$.

{\begin{rem}
Setting $\beta = 1-s\in [0,1[$, condition~\eqref{apriori} is equivalent to the fact that {for all $x,y\in X$ the series
\begin{equation}\label{eq:Stein3}
K^\beta_\rho(x,y)= \sum_{j\in J} \sigma_j^\beta \phi_j(y)\phi_j(x)
\end{equation}
converges absolutely to a bounded reproducing kernel
$K^\beta_\rho$.
Convergence of the series \eqref{eq:Stein3} was studied e.g.~in \cite{stsc12}, where it is proved that, if the sequence of powers $(\sigma^\beta_j)_{j\in J}$ is summable, there exists a $\rho$-null set $N$ such that \eqref{eq:Stein3} converges absolutely on $(X\setminus N) \times (X\setminus N)$ (see \cite[Proposition 4.4]{stsc12}). We remark that this weaker fact is not sufficient in our setting: indeed, on the one hand it does not imply that the series \eqref{eq:Stein3} (or, equivalently, \eqref{apriori}) converges on all of $X$, and on the other it does not guarantee that such series is uniformly bounded, two conditions which however are both needed in the proof of Theorem \ref{rates} below to get uniform estimates on the whole set $X$.}
A direction of future work is to study the geometric nature of the
above conditions when $X$ is a metric space or a Euclidean space and
$X_\rho$ a Riemannian submanifold.  
\end{rem}}

The following theorem provides the  finite sample  bound on the error $\sup_{x\in X}  \abs{F_n(x)-F_\rho(x)}$.

\begin{thm}\label{rates}
Suppose $\RR(\sigma)=\sigma/(\sigma+\la)$. If Assumption \ref{C} holds and we choose 
$$
\la_n =\left(\frac{1}{n}\right)^{\frac{1}{2s+b+1}}, 
$$
then, for $n\geq 1$ and $\delta > 0$, we have
\begin{equation}\label{eq:rates}
\sup_{x\in X}  \abs{F_n(x)-F_\rho(x)} \leq (C_s\vee( D_b (2\delta\vee\sqrt{2\delta}))) \left(\frac{1}{n}\right)^{\frac{s}{2s+b+1}}
\end{equation}
with probability at least $1-2e^{-\delta}$.
\end{thm}
We postpone the proof to the end of the current section and add here some comments.
The above finite sample bound quantifies the stability of the estimator with respect to random sampling.
Equivalently,  if we set the right hand term of the inequality to $\eps$ and solve for $n=n(\eps, \delta)$, we obtain the sample complexity of the problem, i.e.~how many samples are needed in order to achieve the maximum error $\eps$ with confidence $1-2e^{-\delta}$. 
As remarked before,  Assumption \ref{C}.\ref{C2}) is verified for  $b=1$ by any reproducing kernel. In this limit case 
our result  gives a rate $n^{-s/(2s+2)}$, comparable with the
one that can be obtained inserting \eqref{eq:inpiu} and \eqref{apprx_err} below into inequality \eqref{decomp}, {with $\nor{T_n - T}$ bounded by \eqref{eq:concentration_0}.}

Note that, if  $\dim\ran{T}=N<\infty$, choosing $b=0$, $D_0=\sqrt{N}$, $s=1$ and $C_1 = \max_{j\in J} 1/\sigma_j$, the rate in~\eqref{eq:rates} becomes $n^{-1/3}$.

The proof of Theorem \ref{rates} follows the ideas in \cite{cade07} and is based on refined estimates of the sample 
 and approximation errors. 
 The techniques in \cite{cap06} should allow to derive similar results for filters beyond the Tikhonov one.

\begin{prop}\label{prop:sample_err2}
If Assumption \ref{C}.\ref{C2}) holds true,   then, for $n\geq 1$ and $\delta > 0$, we have
$$
\sup_{x\in X} \abs{F_n(x)-G_{\la_n}(x)}    \leq \left(\frac{\delta}{n\la_n}+\sqrt{\frac{2\delta{\cal N} (\la_n)}{n\la_n}}\right)
$$
with probability at least $1-2e^{-\delta}$. 
\end{prop}
\begin{proof}
Consider the following decomposition
\begin{eqnarray*}
\RRn (T) - \RRn (T_n) &=& (T+\la_n)^{-1}T-(T_n+\la_n)^{-1}T_n\\
&=& (T+\la_n)^{-1}T-(T+\la_n)^{-1}T_n+(T+\la_n)^{-1}T_n-(T_n+\la_n)^{-1}T_n\\
&=& (T+\la_n)^{-1}(T-T_n)+  (T+\la_n)^{-1}[(T_n + \la_n) - (T+\la_n)] (T_n+\la_n)^{-1}T_n\\
&=& (T+\la_n)^{-1}(T-T_n)+  (T+\la_n)^{-1}(T_n-T)  (T_n+\la_n)^{-1}T_n\\
&=& (T+\la_n)^{-1}(T-T_n)[I-  (T_n+\la_n)^{-1}T_n] \\
&=&  \la_n (T+\la_n)^{-1}(T-T_n) (T_n+\la_n)^{-1}.
\end{eqnarray*}
It is easy to see that {$\nor{(T_n+\la_n)^{-1}}_\infty\leq \la_n^{-1}$}, hence
{\[
\nor{\RRn (T) - \RRn (T_n)}_{\mathcal S_2}\leq \la_n\nor{(T+\la_n)^{-1}(T-T_n)}_{\mathcal S_2}\nor{(T_n+\la_n)^{-1}}_\infty \leq \nor{(T+\la_n)^{-1}(T-T_n)}_{\mathcal S_2} .
\]}
Then, from Lemma~\ref{conce2} in the Appendix we have that 
$$
\nor{(T+\la_n I)^{-1}(T-T_n)}_{\mathcal S_2}\leq \left( \frac{\delta}{n\la_n}+\sqrt{\frac{2\delta{\cal N}(\la_n)}{n\la_n}} \right),
$$
with probability at least $1-2e^{-\delta}$, so that  the result follows by \eqref{decomp2}.
\end{proof}

\begin{prop}\label{prop:apprx_err}
If Assumption \ref{C}.\ref{C1}) holds true,  then
\begin{equation}\label{apprx_err}
\sup_{x\in X}\abs{G_\la (x)-F_\rho(x)} \leq \la^{s}C_s . 
\end{equation}
\end{prop}
\begin{proof}
Since $\theta(\sigma)-r_{\la}(\sigma)=\la/(\sigma+\la)$ for all $ \sigma >0$,  we have 
\begin{align*}
  \abs{G_{\la}(x)-F_\rho(x)} & = \abs{\scal{(r_{\la}(T)
      -\theta(T))K_x}{K_x}} = \abs{\scal{(r_{\la}(T) -\theta(T))P_\rho
      K_x}{P_\rho K_x}} \\ & =\la \nor{ (T+\la)^{-\frac12}P_\rho K_x}^2 ,
\end{align*}
as $P_{\rho}K_x \in \ker{T}^\perp$. Since by assumption $P_{\rho}K_x\in \ran{T^{s/2}}$ for some $0<s\leq 1$, spectral calculus and the bound $\sigma^s/(\sigma+\la) \leq \la^{s-1}$ give  the inequality
$$
\nor{ (T+\la)^{-\frac12}P_\rho K_x}^2 = \nor{ [(T+\la)^{-1} T^s]^{\frac 12} T^{-\frac s2} P_\rho K_x}^2 \leq \la^{s-1}
  \nor{T^{-\frac s2}P_\rho K_x}^2 ,
$$
so that 
\[
  \abs{G_{\la}(x)-F_\rho(x)} \leq \la^{s} \nor{T^{-\frac s2}P_\rho K_x}^2\leq \la^sC_s
\]
for all $x\in X$.
\end{proof}
We are now  ready to prove the main result.
\begin{proof}[Proof of Theorem \ref{rates}]
The choiche $\la_n=n^{-1/(2s+b+1)}$ is the one that set  the contributions of the sample and approximation errors in \eqref{decomp} to be equal.
Indeed, we begin by simplifying the bound on the sample error. If $\la \geq n^{-1}$, then $n\la\geq\sqrt{n\la^{b+1}}$ for all $0<b\leq1$, so that 
$$
\frac{\delta}{n\la }+\sqrt{\frac{2\delta{\cal N}(\la)}{n\la }} = \frac{\delta}{n \la}+\sqrt{\frac{2\delta {\mathcal N}(\la)\la^b}{n\la^{b+1} }}
       \leq D_b(\delta\vee\sqrt{2\delta}) \left(\frac{1}{n\la}+\frac{1}{\sqrt{n\la^{b+1}}}\right)
       \leq  \frac{2D_b(\delta\vee\sqrt{2\delta})}{\sqrt{n\la^{b+1}}},
$$
where we used the definition of $D_b$ (and the fact that $D_b\geq 1$). Then, by the above inequality and Propositions \ref{prop:sample_err2} and \ref{prop:apprx_err}, inequality \eqref{decomp} gives 
\begin{equation}\label{staqua}
\sup_{x\in X}  \abs{F_n(x)-F_\rho(x)} \leq C_s\la^s+ \frac{2D_b(\delta\vee\sqrt{2\delta})}{\sqrt{n\la^{b+1}}}.
\end{equation}
If we set  the contributions of the sample and approximation errors to be equal, the choice for $\la$ is
$$
\la =\left(\frac{1}{n}\right)^{\frac{1}{2s+b+1}}.
$$
It is easy to see that $\la \geq n^{-1}$ for all  values of $s,b$, so that  from \eqref{staqua} we have 
$$
\sup_{x\in X}  \abs{F_n(x)-F_\rho(x)} \leq (C_s\vee (2D_b (\delta\vee\sqrt{2\delta}))) \left(\frac{1}{n}\right)^{\frac{s}{2s+b+1}} .
$$
\end{proof}

\subsection{The kernel PCA filter}
A natural choice for the spectral filter $\RR$ would be the regularization defined by kernel PCA \cite{scsmmu98},
 that  corresponds to truncating  the generalized inverse of the kernel matrix at some cutoff parameter $\la$.  The corresponding filter function is 
\[ \RR(\sigma)=
\begin{cases}
  1 & \sigma\geq \la \\
  0 & \sigma <\la
\end{cases}.\]
The above filter does not satisfy   the Lipschitz condition~\ref{B}.\ref{B3}) in Assumption~\ref{B}, so that the bound \eqref{eq:inpiu} for the sample error $\sup_{x\in X} \abs{F_n (x) - G_{\la_n} (x)}$ does not hold in this case\footnote{Note that, by Proposition \ref{prop:misurabilita2} in~\ref{maurer}, if $X$ is locally compact, then $F_n$ defined in \eqref{estimator} still is a $C(X)$-valued estimator.}. However, 
{{we can still achieve an estimate by employing inequality \eqref{ineMaurer_bis} in~\ref{maurer2}. To this aim,}
with a slight abuse of the notation, here we count the eigenvalues of $T$ and $T_n$ without their multiplicities and we list them in decreasing order. Furthermore, for any $\la>0$ we set $\sigma_{j(\la)}$ and $\sigma_{k(\la)}^{(n)}$ as the smallest eigenvalues of $T$ and $T_n$ which are greater or equal to $\lambda$, i.e.
\[ \sigma_1>\sigma_2>\ldots>\sigma_{j(\la)}\geq\la>\sigma_{j(\la)+1}\qquad \sigma^{(n)}_1>\sigma^{(n)}_2>\ldots>\sigma_{k(\la)}^{(n)}\geq\la>\sigma_{k(\la) +1}^{(n)}. \]
Inequality~\eqref{ineMaurer_bis} implies that
$$
\nor{ r_\la (T_n)-r_\la (T)}_{\mathcal S_2} \leq
\frac{\nor{T_n-T}_{\mathcal S_2}}{\min\set{\sigma_{j(\la)}-\sigma_{k(\la)+1}^{(n)}, \sigma_{k(\la)}^{(n)}-\sigma_{j(\la)+1}}} \leq \frac{\nor{T_n-T}_{\mathcal S_2}}{\min\set{\sigma_{j(\la)}-\lambda, \lambda-\sigma_{j(\la)+1}}} ,
$$
and} inequality \eqref{decomp2} for the sample error then reads
$$
\sup_{x\in C} \abs{F_n(x)-G_{\la_n}(x)} \leq \frac{\nor{T_n-T}_{\mathcal S_2}}{\min\set{\sigma_{j(\la_n)}-\lambda_n , \lambda_n - \sigma_{j(\la_n)+1}}} .
$$
By Lemma \ref{concentration}, in order to have convergence to $0$ of the right hand side of this expression we need to choose the sequence $(\la_n)_{n\geq 1}$ such that
$$
\sup_{n\geq 1} \frac{\log n}{\sqrt{n} \min\set{\sigma_{j(\la_n)}-\lambda_n , \lambda_n - \sigma_{j(\la_n)+1}}} < \infty . 
$$
Since the gap $\sigma_{j(\la)}-\sigma_{j(\la)+1}$ can have any arbitrary rate of convergence to zero as $\la\to 0^+$, we thus see that there exists {\em no} distribution independent choice of $(\lambda_n)_{n\geq 1}$ ensuring the convergence to zero of the above bound.

Note that $r_\la (T)$ is the projection $P_{j(\la)}$ onto the sum of the eigenspaces of
the first $j(\la)$ eigenvalues of $T$ and $r_\la(T_n)$ is the
projection $P^{(n)}_{k(\la)}$ onto the sum of the eigenspaces of
the first $k(\la)$ eigenvalues of $T$. If $(M_n)_{n\geq 1}$ is any strictly increasing sequence with $M_n \in \N$ for all $n$, we can consider the following  distribution dependent choice $\lambda_n=(\sigma_{M_n}+\sigma_{M_n+1})/2$. Then we have 
\[
\nor{ P^{(n)}_{M_n}-P_{M_n}}_{\mathcal S_2}= \nor{r_{\la_n}(T_n)-r_{\la_n}(T)}_{\mathcal S_2} \leq \frac{2\nor{T_n-T}_{\mathcal S_2}}{\sigma_{M_n}-\sigma_{M_n+1}},
\]
which recovers a known result  about kernel PCA (see for example
\cite{zwbl06}). Furthermore, if the bound $\nor{T_n-T}_{\mathcal 
  S_2}<(\sigma_{M_n}-\sigma_{M_n+1})/2$ holds,  then we obtain $\nor{P^{(n)}_{M_n}-P_{M_n}}_{\mathcal S_2}<1$, hence we have the equality 
$\dim\ran{P^{(n)}_{M_n}}=\dim\ran{P_{M_n}}$.

The following result extends Theorem~\ref{thm:consistency} to the case of  kernel PCA, 
at the price of  having  a distribution dependent choice of the cut-off sequence $(M_n)_{n\geq 1}$.

\begin{thm}\label{thm:consistencyPCA} If the sequence  of natural numbers $(M_n)_{n\geq 1}$ is strictly  increasing and such that
\begin{equation*}
\sup_{n\geq 1} \frac {\log n}{\sqrt{n}(\sigma_{M_n}-\sigma_{M_n+1})}<+\infty
\end{equation*}
and we define the sequence $(\la_n)_{n\geq 1}$ as
$$
\lambda_n=\frac{\sigma_{M_n}+\sigma_{M_n+1}}{2} ,
$$
then,  for every  compact subset $C\subset X$, 
\begin{equation*}
  \lim_{n\to\infty} \sup_{x\in C} \lvert F_n(x)-F_\rho(x)\rvert=0
  \qquad\text{almost surely}.
  \end{equation*}
\end{thm}
\begin{proof}
By the above discussion and inequality \eqref{decomp2},
$$
\sup_{x\in C} \abs{F_n(x)-G_{\la_n}(x)} \leq \frac{2\nor{T_n-T}_{\mathcal S_2}}{\sigma_{M_n}-\sigma_{M_n+1}} \leq \frac{\sqrt{n} \nor{T_n-T}_{\mathcal S_2}}{\log n} \, \sup_{n\geq 1} \frac {2\log n}{\sqrt{n}(\sigma_{M_n}-\sigma_{M_n+1})} .
$$
Convergence to $0$ of the sample error then follows from \eqref{eq:concentration}. Combining this fact and Proposition \ref{appr_err} into inequality \eqref{decomp}, the claim then follows.
\end{proof}

\section{Some Perspectives}\label{sec:disc}

In this section we discuss some different perspectives to our approach and suggest some possible extensions.

\subsection{Connection to Mercer Theorem}\label{sec:Mercer}
We start discussing some connections between our analytical
characterization of the support of $\rho$ {and Mercer} theorem
\cite{mer09}. With the notations of Section~\ref{sec:int}, the fact
that the family {$(\sqrt{\sigma_j}\phi_j)_{j\in J}$ is an
  orthonormal basis of $P_\rho\hh$ and  the reproducing property
 give the relation
\begin{equation}\label{eq:mercer}
\scal{P_\rho K_{\yps}}{K_x}= \sum_{j\in J} \sigma_j \phi_j (x) {\phi_j (\yps)} \qquad \forall x,y\in X ,
\end{equation}}
where  the  series converges absolutely. Note that in this expression the eigenfunctions $\phi_j$ of $L_K$ are defined outside $X_\rho$ through the  extension equation~\eqref{eq:outofsample}. Restricting \eqref{eq:mercer} to $x,y\in X_\rho$, we obtain
$$
K(x,y)=\sum_{j\in J} \sigma_j \phi_j (x) {\phi_j (\yps)} \qquad \forall x,y\in X_\rho ,
$$
which is nothing else than Mercer theorem \cite{stch08}. In particular, taking $x=y$, this formula implies that 
$\sum_{j\in J} \sigma_j \abs{\phi_j (x)}^2 =K(x,x)$
 for all $x\in X_\rho$. On the other hand, the assumption that the reproducing kernel separates $X_\rho$ precisely ensures that
$$\sum_{j\in J} \sigma_j \abs{\phi_j (x)}^2 \neq K(x,x)\qquad
\forall x\not\in X_\rho.$$
(Recall that, if $K$ separates $X_\rho$, then $X_\rho$ is the $1$-level set of the function $F_\rho= \sum_{j\in J} \sigma_j \abs{\phi_j}^2 $.{)}

\subsection{A Feature Space Point of View}

In machine learning,  kernel methods are often described in terms of a corresponding feature map \cite{vapnik98}.
This point of view highlights the linear structure of the Hilbert space and often provides a more geometric interpretation.

We recall that a feature map associated to a reproducing kernel is a map $\Psi:X\to{\mathcal F}$, where ${\mathcal F}$ is a Hilbert space with inner product $\scal{\cdot}{\cdot}_{\mathcal F}$, satisfying
 $ K(x,\yps)=\scal{\Psi(\yps)}{\Psi(x)}_{\mathcal F}. $
While every map $\Psi$ from $X$ into a Hilbert space $\cal F$ defines a reproducing kernel, it is also possible to prove that each kernel has an associated  feature map (and in fact many). 
Indeed, given $K$, the natural assignment is $\ff\equiv\hh$ and
$\Psi(x)\equiv { K_x}$. Such a choice is also minimal, in the
sense that, if we make a different choice of $\ff$ and $\Psi$, then
there exists an isometry $W:\hh\to\ff$ such that $\Psi(x) = W {K_x}$
$\forall x\in X${ -- see for example Proposition 2.4 of \cite{cadeto06} or Theorem~4.21 of \cite{stch08}, noticing that both papers deal with the transpose $W^\top:\ff\to\hh$.}

We next review some of the concepts introduced in Section \ref{sec:basic} in terms of feature maps.
 For the sake of comparison we assume that $\nor{\Psi(x)}_{\cal F}=1$
 for all $x\in X$ (this corresponds to the normalization assumption \ref{A}.\ref{A4})), we let ${\cal F}_C$ be the closure of the linear span
 of the set $\set{\Psi(x)\mid x \in C}$, and define 
$$
d_{\cal F}(\Psi(x), {\cal F}_C) = \inf_{f\in {\cal F}_C} \nor{\Psi(x)-f}_{\cal F}.
$$
It is easy to see that the definition of separating kernel has the following equivalent and natural analogue in the context of feature maps. 
\begin{defn}\label{Phiseparated} 
We say that  a feature map $\Psi$
 separates   a subset $C\subset X$ if   
\begin{eqnarray*}
 d_{\cal F}(\Psi(x), {\cal F}_C) =0 \quad\Longleftrightarrow\quad x~\in ~C.
\end{eqnarray*}
\end{defn}

The above definition is equivalent to Definition \ref{separated}  since 
$d_{\cal F}(\Psi(x), {\cal F}_C)=\nor{\Psi(x)-Q_C\Psi(x)}_{\cal F}$, where $Q_C$ is the orthogonal projection onto ${\cal F}_C$. 
Then, according to Definition \ref{Phiseparated}, a point $x\in C$ if and only if $\nor{\Psi(x)-Q_C\Psi(x)}_{\cal F}^2=0$. Since $\Psi(x) = WK_x$ $\forall x\in X$ and $Q_C W = WP_C$, this is equivalent to
$$
0 = \nor{\Psi(x)-Q_C\Psi(x)}_{\cal F}^2 = \nor{K_x - P_C K_x}^2 = K(x,x) - F_C (x) .
$$
Theorem \ref{prop_proj} then implies that Definition \ref{separated} and \ref{Phiseparated} are equivalent.
We thus see that the  separating property  has a clear  geometric
interpretation in the  feature space: the set $\Psi(C)$ is the
intersection of the closed subspace $\mathcal F_C$, i.e.~a linear manifold in $\cal F$, 
and $\Psi(X)$ -- see Figure~\ref{Fig:PhiSep}. 

In the above interpretation, the estimator we propose for the support
then stems from the following observation: given a training set
$x_1,\ldots , x_n$, we
classify a new point $x$ as belonging to the estimator $X_n$ of
$X_\rho$ if the distance of $\Psi(x)$ to the linear span of $\{ \Psi(x_1),
\ldots \Psi(x_n)\}$ is sufficiently
small. 

Given a training set $\set{x_1,\ldots,x_n}$, our estimator $F_n$
  classifies a new point $x$ as belonging to the support if the
  distance of $\Psi(x)$ to the linear span of
  $\Psi(x_1),\ldots,\Psi(x_n)$ is sufficiently small. 

\begin{figure}[t!]
\begin{center}
\includegraphics[width=5in]{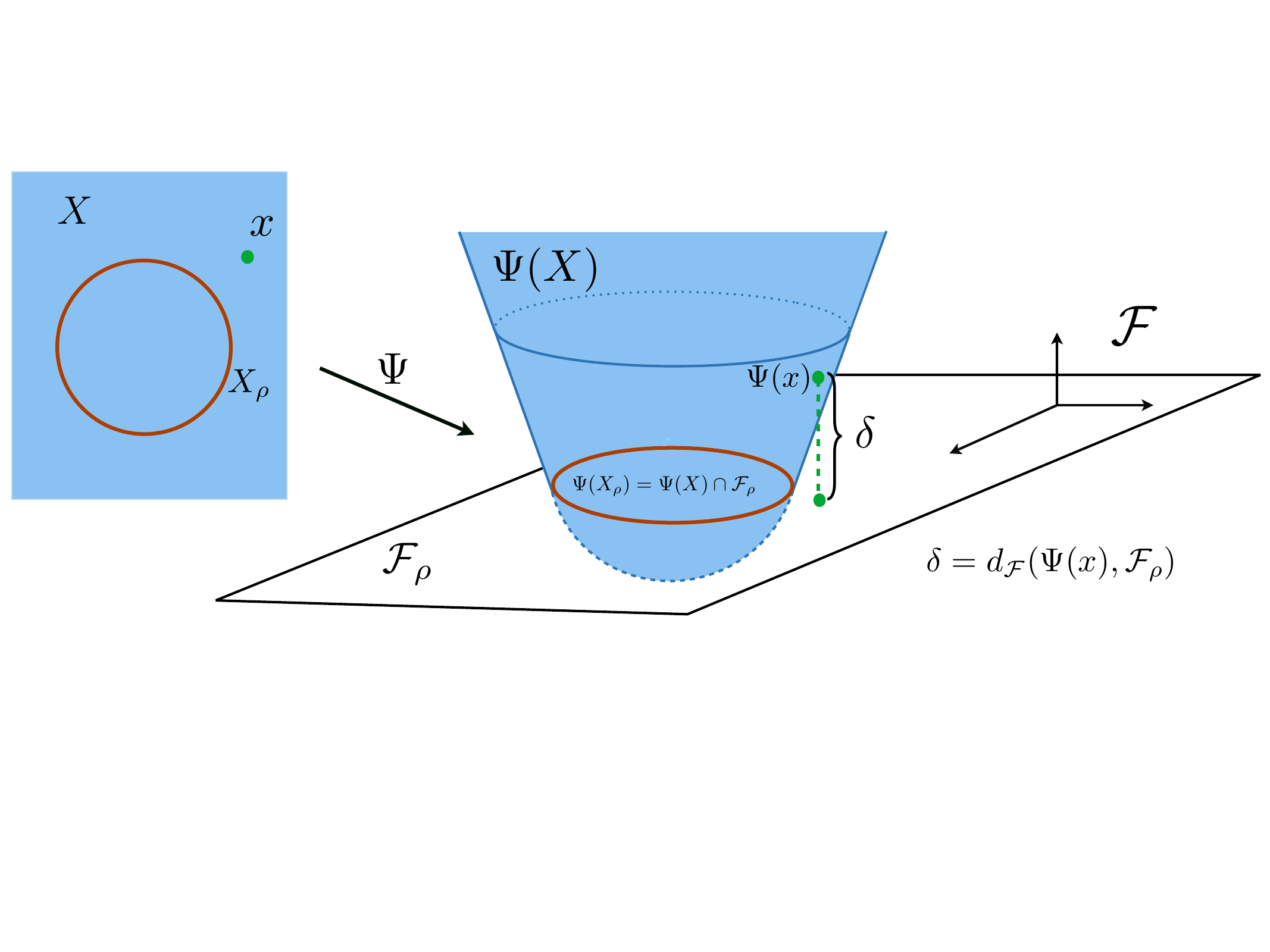}
\caption{The sets $X$  and the support $X_\rho$ are mapped into the feature space $\cal F$, by the feature map $\Psi$. 
Here we take ${\cal F}_{\rho}={\cal F}_{X_\rho}$ to be a linear space passing through the origin. The image of the support  
with respect to the feature map is  given by the intersection of the image of $X$ with ${\cal F}_{\rho}$. 
By the separating property, a point $x$ belongs to the support if and only the distance {between $\Psi(x)$ and $\mathcal F_\rho$} is zero.}\label{Fig:PhiSep}
\end{center}
\end{figure}

\subsection{Inverse Problems and Empirical Risk Minimization}\label{IPERM}

Here we suggest a simple interpretation of the estimator $F_n$ and
stress the connection with the supervised setting. 
We regard the sampled data $x_1,\ldots,x_n$ as a training set of positive examples, so
that each point $x_i\in X_\rho$ almost surely; the new datum is the point $x\in X$,
and we evaluate the estimator $F_n$ at $x$. We label the examples according to the similarity function $K$ by setting
\[
y_i(x)=K(x_i,x) \equiv ({\mathbf K}_x)_i \qquad i=1,\ldots,n .
\]
If $K$ satisfies Assumption \ref{A}, then, since $K(x,x)=1$ and $K$ is $\dk$-continuous, the function $y_i$ is close to $1$
whenever $x_i$ is close to $x$. The interpolation problem
\[
\text{find } f\in\hh \text{ such that }
f(x_i)=y_i(x) \ \forall i\in\{1,\ldots, n\} \quad\Longleftrightarrow\quad  S_n f = {\mathbf K}_x
\]
(where $S_n$ is defined in \eqref{eq:def_Sn}) is ill-posed.
To restore well-posedeness we can consider  the corresponding least square problem (empirical risk minimization problem)
$$
\min_{f\in \hh}\frac 1 n \sum_{i=1}^n \abs{f(x_i)-y_i(x)}^2 \quad\Longleftrightarrow\quad \min_{f\in \hh} \frac 1 n \nor{S_n f-{\mathbf K}_x}^2_{{\R^n}} ,
$$
or in fact  its regularized version
$$
\min_{f\in \hh}\left(\frac 1 n \sum_{i=1}^n \abs{f(x_i)-y_i(x)}^2+\la\nor{f}^2\right) \quad\Longleftrightarrow\quad \min_{f\in \hh} \left(\frac 1 n \nor{S_n f-{\mathbf K}_x}^2_{{\R^n}}+\la\nor{f}^2\right) ,
$$
where $\la>0$ is the regularization parameter (Tikhonov regularization). It is known \cite{enhane} that the minimum of the above expression is achieved by $f\equiv f_n^\la$, with
\[
f_n^\la= \frac{1}{n}g_\la(\frac{S_n^{\top}S_n}{n})S^{\top}_ny ,
\]
where $\G$ is the function $\G(\sigma) = 1/(\sigma + \la)$.\\
More generally, Tikhonov regularization can
be replaced by spectral regularization induced by a different choice of the filter $\G$; the corresponding regularized solution $f_n^\la$ is still given by the previous equation, but the function $\G$ appearing in it is now completely arbitrary. Comparing with \eqref{eq:estimator2}, we see that $f^{\la_n}_n(x)=F_n(x)$. Equation \eqref{set_est} has then the following interpretation: a new point $x$  is estimated to be a positive example (that is, to belong to the support $X_\rho$) if and only   if $f^{\la_n}_n(x) \geq 1-\tau$, where $\tau$ is a threshold parameter.

The above discussion  suggests several extensions and 
variations of our method, obtained considering more general  penalized empirical risk minimization functionals of the form 
$$
\min_{f\in \hh}\left(\frac 1 n \sum_{i=1}^n V(y_i(x), f(x_i))+\la R(f)\right) ,
$$
where:
\begin{itemize}
\item $V$ is a (regression) loss function measuring the approximation property of $f$, for example the logistic loss or a robust loss such as the one used in support vector machine regression. Our theoretical analysis does not carry on to  other loss functions and different mathematical concepts from empirical process theory are probably needed;
\item $R$ is a regularizer  measuring the complexity of a function $f\in\hh$. For example, one can consider the case 
where the kernel is given by  a dictionary  of atoms $f_\gamma:X\to {\R}$, with $\gamma \in \Gamma$, 
such that $\sum_{\gamma\in \Gamma} \abs{f_\gamma(x)}^2=1$, so that  we have $K(x,\yps)=\sum_{\gamma\in \Gamma} f_\gamma(x) {f_\gamma(\yps)}$
and, hence, ${f=\sum_{\gamma\in \Gamma} w_\gamma f_\gamma}$, with $w=(w_\gamma)_{\gamma \in \Gamma}\in \ell_2(\Gamma)$. In this setting, Tikhonov regularization corresponds to the choice $R(f)=\sum_{\gamma\in \Gamma} \abs{w_\gamma}^2$, but other norms, such as the  $\ell_1$ norm $\sum_{\gamma\in \Gamma}|w_\gamma|$,  
can also be considered.
\end{itemize}

\section{Empirical Analysis}\label{sec:emp}

In this section we describe some preliminary experiments aimed at testing the properties 
and the performances  of the proposed methods both on simulated and real data. 
We only  discuss spectral algorithms induced by Tikhonov  regularization to contrast the general method to 
some current state of the art algorithms.
Note that while computations can be made more efficient in several ways, we consider a simple algorithmic  
protocol and leave a more refined computational study  for future work. 
Recall that Tikhonov regularization defines an estimator 
$F_n(x)={\mathbf K_x}^\ast ({\mathbf K}_n+n\lambda)^{-1}{\mathbf K}_x$,
and a point $x$ is labeled as belonging to the support $X_\rho$ if $F_n(x)\geq 1-\tau$. 
The computational cost for the algorithm is,  in the worst case, of order $n^3$ -- like standard regularized least squares -- for training, and order $Nn^2$ if we have to  predict the value of $F_n$ at $N$ test points. 
In practice, one has to choose a good value for the regularization 
parameter $\la$ and this requires computing multiple solutions, a so called {\em regularization path}. 
As noted in \cite{RifLip07}, if we form  the inverse  using the eigendecomposition of the 
kernel matrix the price of computing   the full regularization path is essentially the same as 
that of computing a single solution (note that the cost of the eigen-decomposition of ${\mathbf K}_n$ 
is also of order $n^3$, though the constant is worse).  This is the strategy that we consider in the following.
In our experiments we considered two datasets: the MNIST\footnote{http://yann.lecun.com/exdb/mnist/} dataset and the CBCL\footnote{http://cbcl.mit.edu/} face database.
 For the digits we considered a reduced set consisting of a training set of 5000 images and a test set of 1000 images.
In the first experiment we trained on $500$ images for the digit $3$ and tested on $200$ images of digits $3$ and $8$.  
Each experiment consists of training on one class and testing on two different classes and 
was repeated for 20 trials over different training set choices. 
For all our experiments we considered the Abel kernel. 
Note  that  in this case  the algorithm requires to choose $3$ parameters: 
the regularization parameter $\la$, the kernel width $\sigma$ and the threshold $\tau$.
In supervised learning cross validation is typically used for parameter tuning, but cannot be 
used in  our setting since support estimation is an  unsupervised problem.
Then, we considered the following  heuristics. The kernel width is chosen as the median  
of the distribution of distances of the $k$-th nearest neighbor of each training set point for $k=10$.
Fixed the kernel width, we choose the regularization parameter in correspondence of the maximum curvature 
in the eigenvalue behavior -- see Figure \ref{f:Eigen} -- the rationale being that after this value the eigenvalues are relatively small. 
\begin{figure}[t!]
\begin{center}
\includegraphics[width=8cm]{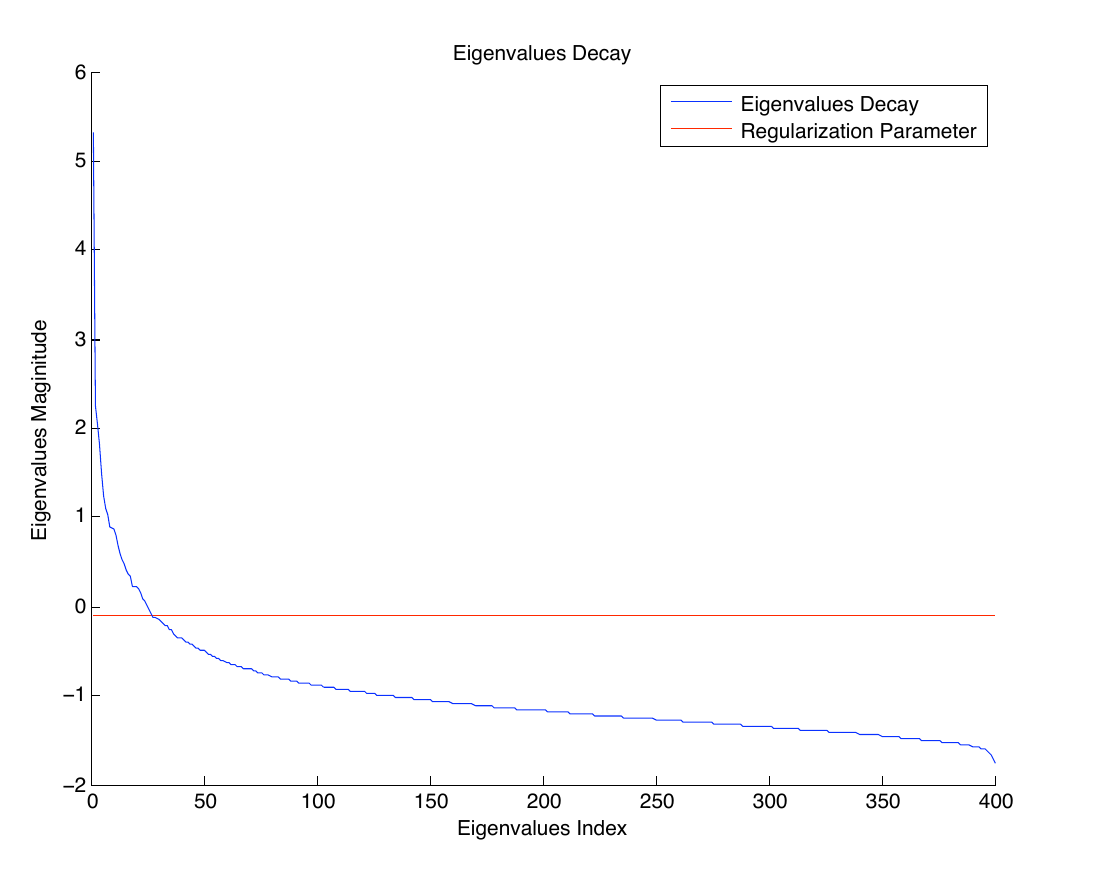}
\end{center}
\caption{{Decay of the eigenvalues of the kernel matrix ordered in decreasing magnitude and corresponding regularization parameter in logarithimic scale.}} 
\label{f:Eigen}
\end{figure}

\begin{figure}
\begin{center}
\begin{minipage}{0.3\textwidth}
\centerline{\tiny{MNIST $9 vs 4$}}
\includegraphics[width=\columnwidth]{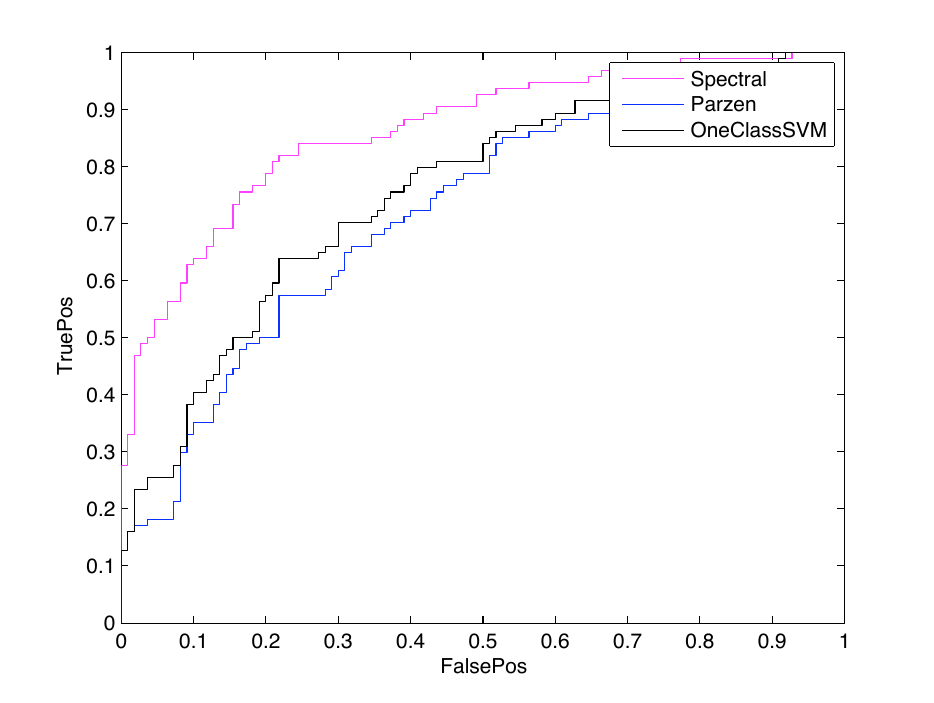}
\end{minipage}
\begin{minipage}{0.3\textwidth}
\centerline{\tiny{MNIST $1 vs 7$ }}
\includegraphics[width=\columnwidth]{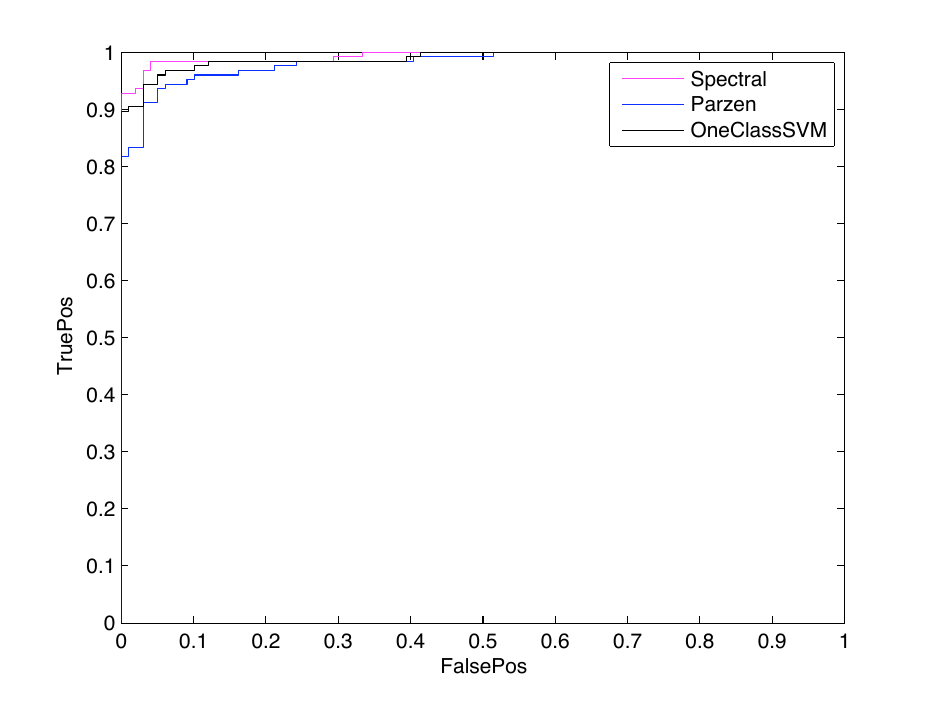}
\end{minipage}
\begin{minipage}{0.3\textwidth}
\centerline{\tiny{CBCL}}
\includegraphics[width=\columnwidth]{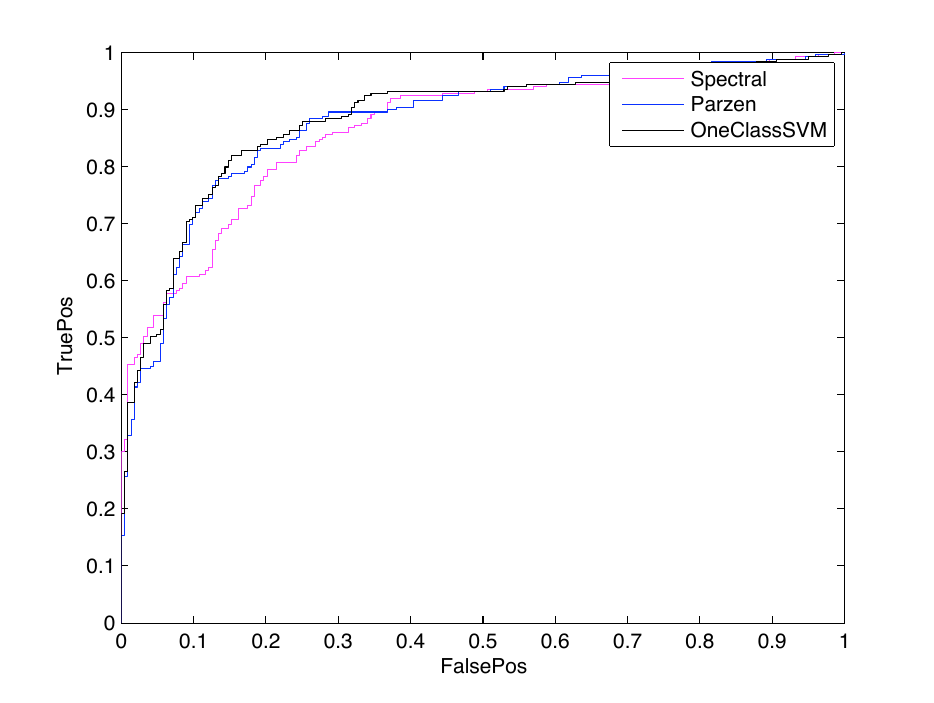}
\end{minipage}
\caption{{ROC curves for the different estimator in three different tasks: digit $9$vs $4$ (Left), digit $1$vs $7$ (Center), CBCL (Right). }}
\label{f:ROC}
\end{center}
\end{figure}
\begin{table}[t!]
\begin{center}
\begin{tabular}{|l|c|c|c|c|c|c|}
\hline
&$3$vs $8$&$8$vs $3$&$1$vs $7$&$9$vs $4$&CBCL\\ 
\hline
\textbf{Spectral}&$0.837\pm0.006$&$0.783\pm0.003$&$0.9921\pm 0.0005 $& $0.865\pm 0.002$&$ 0.868\pm 0.002$\\\hline
\textbf{Parzen}&$0.784\pm0.007$&$0.766\pm0.003$&$0.9811\pm 0.0003 $& $0.724\pm 0.003$&$0.878\pm  0.002 $\\\hline
\textbf{1CSVM}&$0.790\pm 0.006$&$0.764\pm0.003$&$0.9889\pm 0.0002$& $0.753\pm 0.004$ &$0.882\pm 0.002$\\
\hline
\end{tabular}
\end{center}
\caption{Average and standard deviation of the AUC for the different estimators on the considered tasks.}\label{t:AUC}
\end{table}

For comparison we considered a Parzen window density estimator  and one-class SVM  (1CSVM) as implemented  by \cite{SVM-KMToolbox}.
For the Parzen window estimator we used the same kernel of the spectral algorithm, that is the Laplacian kernel, and also the same width.  Given a kernel width, an estimate of the probability distribution is computed and can be used to 
estimate the support by fixing a threshold $\tau'$.   For the one-class SVM we considered the  Gaussian 
kernel, so that we have  to fix the kernel width and a regularization parameter $\nu$. We fixed the kernel width to be the same 
used by our estimator and set $\nu=0.9$. For the sake of comparison,  also for one-class SVM  we considered a varying offset $\tau^{''}$.
The performance is evaluated computing   ROC curve (and the corresponding AUC value) for varying values of the thresholds  $\tau,\tau',\tau^{''}$. 
The ROC curves  on  the different tasks are reported (for one of the trials) in Figure \ref{f:ROC},
Left. 
The mean and standard deviation of the  AUC for the three methods is reported in Table \ref{t:AUC}. 
Similar experiments were repeated considering other pairs of digits, see Table \ref{t:AUC}. 
Also in the case of the CBCL datasets we considered a reduced dataset consisting of $472$ images for training and other $472$ for test. 
On the different test performed on the MNIST data the spectral algorithm always achieves results which are  better -- and often  substantially better -- than those of the other methods. On the CBCL dataset SVM provides the best result, but spectral algorithm still provides a competitive  performance.  

{
\begin{rem}
 We remark that, although binary classification data sets are used in the experiments, the considered set-up
 is that of a one-class classification problem.   Indeed,  the   training and tuning of the algorithms are performed using only examples of one class and the other class is only considered for testing.
 Accordingly, the proposed methods are compared to state of the art algorithms for one-class 
 classification.
%
\end{rem}}

\appendix

\section{Auxiliary Proofs}

In this section we give the proofs of a few technical results needed in the paper.

{\subsection{Normalizing a Kernel}\label{app:lemma}}

{The next result shows that, if $K$ is a reproducing kernel which is nonzero on the diagonal, then it can be normalized, and its normalized version separates the same sets. When $K(x,x) = 0$ for some $x\in X$, then clearly this result still holds replacing the set $X$ with $X\setminus X_0$ and considering the restriction of $K$ to $(X\setminus X_0)\times (X\setminus X_0)$, where $X_0 = \set{x\in X\mid K(x,x) = 0}$.
\begin{prop}\label{l:normalization}
 Assume that $K(x,x)>0$ for
 all $x\in X$. Then, the reproducing kernel $K'$ on $X$, given by 
\[
K'(x,\yps)=\dfrac{K(x,\yps)}{\sqrt{K(x,x)K(\yps,\yps)}} \qquad \forall x,\yps\in X ,
\]
is normalized and separates the same sets as $K$.
\end{prop}
\begin{proof} Clearly $K$ is a kernel of positive type.
Denote by $\hh'$ the reproducing kernel Hilbert space with kernel
  $K'$, and define the feature map $\Psi: X\to \hh$,
  $\Psi(x)=K_x/\nor{K_x}$. It is simple to check that 
  $\scal{\Psi(y)}{\Psi(x)}=K'(x,y)$ and $\Psi(X)^\perp=\{0\}$, so that the map $\Psi_*:\hh\to\hh'$
  \[ (\Psi_*f)(x) =\scal{f}{\Psi(x)}\] 
is a unitary operator  with $K'_x=\Psi_*( \Psi(x))$ \cite{cadeto06}.  Clearly, for any $f\in \hh$ and $x\in X$
\[
\scal{\Psi_* f}{K'_x} = \scal{\Psi_* f}{\Psi_*
\Psi(x)} = \frac{\scal{f}{K_x}}{\nor{K_x}}.
\]
The above  equality shows that $\hh$ and $\hh'$ separate the same sets. 
\end{proof}}

\subsection{Analytic Results}\label{maurer}

In this section, we suppose that the kernel $K$ satisfies Assumption~\ref{A}, and endow the set $X$ with the metric { $\dk$} induced by $K$. {Measurability of a map taking values in a topological space will be always understood with respect to the Borel $\sigma$-algebra of such space.} The next simple lemma will be used frequently.
\begin{lem}\label{lem:app1}
For all $k=1,2$, the map
$$
\xi:X \to \mathcal S_k , \qquad \xi(x)=K_x\otimes K_x
$$
is continuous and measurable. Moreover, if $Z_i : \Omega \to \mathcal S_k$ is given by
$$
Z_i (\omega) = K_{x_i} \otimes K_{x_i} \qquad \omega = (x_j)_{j\geq 1} ,
$$
then $Z_i$ is measurable for all $i\geq 1$.
\end{lem}
\begin{proof}
The map ${x \mapsto K_x}$, is continuous {from $X$ into $\hh$} by item
\ref{P.1.i}) in Proposition \ref{metrica}. Since {$\xi(x) = K_x\otimes K_x$}, continuity of $\xi$ follows at once. By item \ref{P.1.iv}) in Proposition \ref{metrica}, $\xi$ is then a measurable map, hence $Z_i$ is such.
\end{proof}

We recall some basic properties of the operator $T$ defined by the kernel. The next result is known (see for example \cite{dedero09}), but we report a short proof for completeness.

\begin{prop}\label{Tprop}
The $\mathcal S_1$-valued map $\xi$ {defined in Lemma~\ref{lem:app1}} is Bochner-integrable with respect to $\rho$, and its integral
\[
T = \int_{X} K_x\otimes K_x d\rho(x)
\] 
is a positive trace class operator on $\hh$, with $\nor{T}_{\mathcal S_1} = \tr{T} = 1$.
\end{prop}
\begin{proof}
The map $\xi$ isbounded because $\nor{K_x\otimes K_x}_{\mathcal S_1} = \tr{K_x\otimes K_x} =
K(x,x) = 1$ and measurable by Lemma \ref{lem:app1} . Therefore, $\xi$ is a Bochner-integrable $\mathcal S_1$-valued map, and
its integral $T$ is a trace class operator. As
$\xi(x)$ is a positive operator for all $x$,  so is $T$.  In particular, $\nor{T}_{\mathcal S_1} = \tr{T}$, and $\tr{T} = \int_X \tr{K_x\otimes K_x}
 d\rho(x)=1$.
\end{proof}

Now, we come to the proof of Proposition \ref{prop:misurabilita}. We will split it into the proofs of Propositions \ref{prop:misurabilita1} and \ref{prop:misurabilita2} below.
\begin{lem}\label{lem:app2}
For all $k=1,2$, the map
$$
\check{T}_n : X^n \to {\mathcal S}_k , \qquad \check{T}_n (x_1,\ldots,x_n) {=} \frac{1}{n} \sum_{i=1}^n K_{x_i} \otimes K_{x_i}
$$
is continuous and measurable.
\end{lem}
\begin{proof}
Evident by Lemma \ref{lem:app1}. 
\end{proof}
\begin{prop}\label{prop:misurabilita1}
For all $n\geq 1$, the map $T_n$ defined in \eqref{eq:Tn} is a $\mathcal S_k$-valued estimator for $k=1,2$.
\end{prop}
\begin{proof}
We have
$$
T_n (\omega) = \check{T}_n (x_1,\ldots,x_n) \qquad \omega = (x_i)_{i\geq 1} ,
$$
hence $T_n$ is measurable by Lemma \ref{lem:app2}.
\end{proof}

For the next proposition we recall that the topology of uniform convergence on compact subsets of $X$ is generated by the following basis of open sets $U_{f,\eps,C}\subset C(X)$
$$
U_{f,\eps,C} = \set{g\in C(X) \mid \sup_{x\in C} \abs{f(x) - g(x)} < \eps} \qquad f\in C(X) ,\, \eps > 0 , \, C\subset X \mbox{ compact} .
$$
\begin{prop}\label{prop:misurabilita2}
Suppose $X$ is locally compact. Let $(\RR)_{\lambda>0}$ be a family of functions $\RR:[0,1]\to [0,1]$ such that each $\RR$ is upper semicontinuous. Then, for any sequence of positive numbers $(\la_n)_{n\geq 1}$ and all $n\geq 1$, the map $F_n$ defined in \eqref{estimator} is a $C(X)$-valued estimator, where $C(X)$ is the space of continuous functions on $X$ with the topology of uniform convergence on compact subsets.
\end{prop}
\begin{proof}
Throughout the proof, $n\geq 1$ will be fixed. Let $(\pphi_k)_{k\geq 1}$ be a decreasing sequence of continuous functions $\pphi_k : [0,1] \to [0,1]$ such that $\pphi_k (\sigma) \downarrow \RRn(\sigma)$ for all $\sigma\in [0,1]$ (such sequence exists by (12.7.8) of \cite{Dieu2}). Then, by Lemma \ref{lem:app2} and continuity of the functional calculus (see e.g.~Problem 126 in \cite{HalmoProb}), for all $k\geq 1$ the map
$$
\pphi_k (\check{T}_n) : X^n \to {\mathcal S}_0 , \qquad [\pphi_k (\check{T}_n)](x_1,\ldots,x_n) {=} \pphi_k (\check{T}_n (x_1,\ldots,x_n))
$$
is continuos from $X^n$ into the Banach space ${\mathcal S}_0$ of the bounded operators on $\hh$ with the uniform operator norm. Thus, for all $x\in X$, the real function $(x_1,\ldots,x_n) \mapsto \scal{[\pphi_k (\check{T}_n)](x_1,\ldots,x_n)\, K_x}{K_x}$ is continuous on $X^n$, hence is measurable by item \ref{P.1.iv}) of Proposition \ref{metrica}. By spectral calculus and dominated convergence theorem,
for all $\omega = (x_i)_{i\geq 1} $
\begin{align*}
\scal{\RRn (T_n (\omega)) \, K_x}{K_x} & = \scal{\RRn (\check{T}_n (x_1,\ldots,x_n)) \, K_x}{K_x} =
\lim_{k\to\infty} \scal{[\pphi_k (\check{T}_n)](x_1,\ldots,x_n)\, K_x}{K_x} 
\end{align*}
It then follows that, for each $x\in X$, the real function $\omega \mapsto \scal{\RRn (T_n(\omega)) K_x}{K_x}$ is measurable on $\Omega$, being the pointwise limit of measurable functions.\\
We now prove that the map $F_n : \omega \mapsto (x\mapsto \scal{\RRn (T_n(\omega)) K_x}{K_x})$ is measurable from $\Omega$ into the space $C(X)$. {By M2, p.~115 in \cite{la93}, this is equivalent to the measurability of the subsets $F_n^{-1} (U)\subset\Omega$ for all open sets $U\subset C(X)$. Since $X$ is a locally compact separable metric space, the topology of uniform convergence on compact subsets is a separable metric topology on $C(X)$ by (12.14.6.2) in \cite{Dieu2}. By separability of $C(X)$, each open set $U\subset C(X)$ then is the denumerable union of sets of the neighborhood basis $\{U_{f,\eps,C} \mid f\in C(X) ,\, \eps > 0 , \, C\subset X \mbox{ compact}\}$. Hence, it is enough to show that $F_n^{-1} (U_{f,\eps,C})$ is measurable for all $f$, $\eps$ and $C$. We have}
$$
F_n^{-1} (U_{f,\eps,C}) = \set{{\omega\in\Omega}\mid \sup_{x\in C} \abs{f(x) - \scal{\RRn (T_n ({\omega})) \, K_x}{K_x}} < \eps} .
$$
By separability of $X$, there exists a countable set $C_0 \subset C$ such that $\overline{C_0} = C$. A continuity argument then shows that
\begin{align*}
F_n^{-1} (U_{f,\eps,C}) & = \bigcap_{k\geq 1} \set{{\omega\in\Omega}\mid \sup_{x\in C} \abs{f(x) - \scal{\RRn (T_n ({\omega})) \, K_x}{K_x}} \leq \eps-\frac 1k} \\
& = \bigcap_{k\geq 1} \bigcap_{x\in C_0} \set{{\omega\in\Omega}\mid \abs{f(x) - \scal{\RRn (T_n ({\omega})) \, K_x}{K_x}} \leq \eps-\frac 1k} .
\end{align*}
Since each set $\set{{\omega\in\Omega}\mid \abs{f(x) - \scal{\RRn (T_n ({\omega})) \, K_x}{K_x}} \leq \eps-1/k}$ is measurable in $\Omega$, measurability of the countable intersection $F_n^{-1} (U_{f,\eps,C})$ then follows.
\end{proof}

{We conclude this section with the proof of measurability of the threshold parameters $(\tau_n)_{n\geq 1}$ defined in \eqref{eq:deftaun}.
\begin{prop}\label{prop:misurtaun}
Suppose $X$ is locally compact. Let $(\RR)_{\lambda>0}$ be a family of functions $\RR:[0,1]\to [0,1]$ such that each $\RR$ is upper semicontinuous. Then, for any sequence of positive numbers $(\la_n)_{n\geq 1}$ and all $n\geq 1$, the map $\tau_n$ defined in \eqref{eq:deftaun} is a $\R$-valued estimator.
\end{prop}
\begin{proof}
As $F_n$ depends only on $(x_1,\ldots,x_n)$, it is clear that so does $\tau_n$. It remains to show measurability of $\tau_n$.\\
Given $i\geq 1$, the map $\omega\mapsto x_i$ is measurable by definition of the product $\sigma$-algebra $\A{\Omega}$ on $\Omega$. Moreover, for any $n\geq 1$, the map $F_n$ is measurable from $\Omega$ into $C(X)$ by Proposition \ref{prop:misurabilita2}. Therefore, the map $\Theta_1 : \Omega \to C(X)\times X$, with $\Theta_1 (\omega) = (F_n (\omega) \, , \, x_i)$, is measurable when $C(X)\times X$ is endowed with the product $\sigma$-algebra of the Borel $\sigma$-algebras of $C(X)$ and $X$, respectively.\\
Since $X$ is locally compact, the map $\Theta_2 : C(X)\times X \to \R$, with $\Theta_2(f,x) = f(x)$, is jointly continuous by \cite[Theorem 5, p.~223]{Kelley} and the discussion following it. Thus, $\Theta_2$ is measurable with respect to the Borel $\sigma$-algebras of $C(X)\times X$ and $\R$.\\
The metric spaces $X$ and $C(X)$ are both separable (for $C(X)$, this is (12.14.6.2) in \cite{Dieu2}). By \cite[Proposition 4.1.7]{Dudley}, the product $\sigma$-algebra of the Borel $\sigma$-algebras of $C(X)$ and $X$ then coincides with the Borel $\sigma$-algebra of $C(X)\times X$. Thus, the composition map $\Phi_i = \Theta_2 \Theta_1$, which is $\Phi_i(\omega) = [F_n(\omega)](x_i)$, is measurable.\\
Finally, the map $m(t_1,\ldots,t_n)\mapsto
\min_{1\leq i\leq n} t_i$ is continuous from $\R^n$ into $\R$, so that $\tau_n = 1-m(\Phi_1,\Phi_2,\ldots,\Phi_n)$ is measurable.
\end{proof}}

\subsection{{A Useful Inequality}}\label{maurer2}
The following proof of inequality \eqref{ineMaurer} below is due to A.~Maurer\footnote{http://www.andreas-maurer.eu}. 
\begin{lem}\label{lem:ineMaurer}
Suppose $S$ and $T$ are two {symmetric}  Hilbert-Schmidt operators on $\hh$ with spectrum contained in the interval $[a,b]$, and let $(\sigma_j)_{j\in J}$ and $(\tau_k)_{k\in K}$ be the eigenvalues of $S$ and $T$, respectively. Given a function $r:[a,b]\to\R$, if the constant
$$
L = \sup_{j\in J, k\in K} \abs{\frac{r(\sigma_j) - r(\tau_k)}{\sigma_j - \tau_k}} \qquad (\mbox{with } 0/0 \equiv 0)
$$
is finite, then
\begin{equation}\label{ineMaurer_bis}
\nor{r(S)-r(T)}_{\mathcal S_2}\leq L \nor{S-T}_{\mathcal S_2}.
\end{equation}
In particular, if $r$ is a Lipshitz function with Lipshitz constant $L_r$, then
\begin{equation}\label{ineMaurer}
\nor{r(S)-r(T)}_{\mathcal S_2}\leq L_r \nor{S-T}_{\mathcal S_2}.
\end{equation}
\end{lem}
\begin{proof}
 Let $(f_j)_{j\in J}$ and $(g_k)_{k\in K}$ be the orthonormal bases of eigenvectors of $S$ and
 $T$ corresponding to the eigenvalues $(\sigma_j)_{j\in J}$ and
 $(\tau_k)_{k\in K}$, respectively, which here we list repeated accordingly to their multiplicity. We have
\begin{align*}
 \nor{r(S)-r(T)}_{\mathcal S_2}^2 & = \sum_{j,k} \abs{\scal{ (r(S)-r(T)) f_j}{g_k}}^2 =
 \sum_{j,k} \left(r(\sigma_j)-r(\tau_k)\right)^2 \abs{\scal{f_j}{g_k}}^2 \\
 & \leq L^2 \sum_{j,k} \left(\sigma_j-\tau_k\right)^2 \abs{\scal{f_j}{g_k}}^2= L^2  \sum_{j,k} \abs{\scal{
   (S-T) f_j}{g_k}}^2 \\
& = L^2 \nor{S-T}_{\mathcal S_2}^2 ,
\end{align*}
which is \eqref{ineMaurer_bis}.
\end{proof}

\subsection{Concentration of Measure  Results}\label{sec:conc}

We will use the following  standard concentration inequality for Hilbert space random variables  (see  Theorem 8.6 in \cite{Pin94}, and \cite{pin99}). Let $\vv$ be a separable Hilbert space and $(\Omega,\A{\Omega},{\mathbb P})$ a probability space. Suppose that $Y_1, Y_2, \ldots$ is a sequence of independent $\vv$-valued random variables $Y_i: \Omega \to \vv$. If
$\EE{\nor{Y_i}_\vv^m}\leq (1/2)m!B^2L^{m-2}$  $\forall m\geq 2$,
then, for all $n\geq 1$ and $\eps >0$,
\begin{equation}\label{ineq_prob}
\PP{\nor{\frac{1}{n}\sum_{i=1}^n Y_i}_\vv > \eps}\leq
2e^{-\frac{n\eps^2}{B^2+L\eps+B\sqrt{B^2+2L\eps}}}.
\end{equation}
We will need in particular the next two straightforward consequences of this inequality.
\begin{lem}\label{conc_lemma}
If $Z_1, Z_2, \ldots$ is a sequence of {i.i.d.} $\vv$-valued random variables, such that $\nor{Z_i}_\vv\leq M$ almost surely, 
$\EE{Z_i} =\mu$ and $\EE{\nor{Z_i}_\vv^2} \leq \sigma^2$ for all $i$, then, for all $n\geq 1$ and $\delta > 0$,
\begin{equation}\label{bernie}
\nor{\frac1 n \sum_{i=1}^n Z_i - \mu}_\vv \leq \frac{M\delta}{n}+
  \sqrt{ \frac{2\sigma^2\delta}{n}}
 \end{equation}
with probability at least $1-2 e^{-\delta}$.
\end{lem}
\begin{proof}
Let $Y_i=Z_i-\mu$. Then $\nor{Y_i}_\vv\leq 2M$ and   $\EE{\nor{Y_i}_\vv^2} \leq  \EE{\nor{Z_i}_\vv^2}=\sigma^2$.
Moreover,  for all $i$ and $m\geq 2$ $\EE{\nor{Y_i}_\vv^m}\leq\sigma^2 (2M)^{m-2}\leq (1/2)m!
 \sigma^2 M^{m-2}$, where the last inequality
 follows since  $2^{m-2}\leq m!/2$.  Then,  
\[ \PP{\nor{\frac{1}{n}\sum_{i=1}^nZ_i-\mu}_\vv > \eps} = \PP{\nor{\frac{1}{n}\sum_{i=1}^n Y_i}_\vv > \eps} \leq
2e^{-\frac{n\eps^2}{\sigma^2+M\eps+\sigma\sqrt{\sigma^2+2M\eps}}}=2
e^{-\frac{\sigma^2n}{M^2}g(\frac{M\eps}{\sigma^2})}=2e^{-\delta},\]
where $g(t)=t^2/(1+t+\sqrt{1+2t})$. \\
Since $g^{-1}(t)=t+\sqrt{2t}$, by solving the equation $(\sigma^2 n /M^2) g(M\eps/\sigma^2)=\delta$ we have
\[
\eps= \frac{\sigma^2}{M} \left(\frac{M^2\delta}{n\sigma^2}+
  \sqrt{ \frac{2M^2\delta}{n\sigma^2}}\right)= \frac{M\delta}{n}+
  \sqrt{ \frac{2\sigma^2\delta}{n}}.
\]
\end{proof}
The above result and  Borel-Cantelli  lemma imply that 
$$
\lim_{n\to \infty} \nor{\frac1 n \sum_{i=1}^n Z_i - \mu}_\vv =0
$$
almost surely. In the paper we actually need a slightly stronger result which is given in the following lemma.
\begin{lem}\label{cor_conc}
If $Z_1, Z_2, \ldots$ is a sequence of {i.i.d.}  $\vv$-valued random variables, such that $\nor{Z_i}_\vv\leq M$ almost surely, then we have
\begin{equation*}\label{AS}
\lim_{n\to \infty} \frac{\sqrt{n}}{\log n} \nor{\frac1 n \sum_{i=1}^n Z_i - \mu}_\vv = 0
\end{equation*}
almost surely.
\end{lem}
\begin{proof}
We continue with the notations in the proof of Lemma \ref{conc_lemma}. By \eqref{ineq_prob}, for all $\eps >0$ we have
\[
\PP{\frac{\sqrt{n}}{\log n} \nor{\frac1 n \sum_{i=1}^n Z_i - \mu}_\vv > \eps} = \PP{\nor{\frac{1}{n}\sum_{i=1}^n Y_i}_\vv > \eps\,\frac{\log
    n}{\sqrt{n}}}\leq
2e^{-A(n,\eps)}=2\left(\frac{1}{n}\right)^{\frac{A(n,\eps)}{\log
    n}},
\]
with 
\[
A(n,\eps) {=} \frac{\eps^2\log^2
  n}{\sigma^2+M\eps\frac{\log
    n}{\sqrt{n}}+\sigma\sqrt{\sigma^2+2M\eps\frac{\log n}{\sqrt{n}}}}.
\]
It follows that
\[
\sum_{n\geq 1} \PP{\frac{\sqrt{n}}{\log n} \nor{\frac1 n \sum_{i=1}^n Z_i - \mu}_\vv > \eps} \leq 2\,\sum_{n\geq
  1}\left(\frac{1}{n}\right)^{\frac{A(n,\eps)}{\log n}}.
\]
For all $\eps>0$,  $\lim_{n\to\infty} A(n,\eps)/\log n =
  +\infty$, so that
the series $\sum_{n\geq
  1} n^{-A(n,\eps)/\log n}$ is
convergent, and Borel-Cantelli lemma gives the result.
\end{proof}
The following inequality is given in \cite{cade07} and we report its proof for completeness.
\begin{lem}\label{conce2}
If Assumption~\ref{A} holds true, then for all $\delta > 0$ we have
\begin{equation*}\label{degree_conc}
\nor{(T+\la)^{-1}(T-T_n)}_{\mathcal S_2}\leq \left( \frac{\delta}{n\la}+\sqrt{\frac{2\delta{\cal N}(\la)}{n\la}} \right)
\end{equation*}
with probability at least $1-2e^{-\delta}$. 
\end{lem}
\begin{proof}
Let $(\Omega,\A{\Omega},{\mathbb P})$ be the probability space defined at the beginning of Section \ref{sec:misurismi}. For all $i\geq 1$ we define
  the random variable $Y_i:\Omega\to \mathcal S_2 $ as 
\[Y_i(\omega)=(T+\la)^{-1} (K_{x_i}\otimes K_{x_i})\qquad \omega=(x_j)_{j\geq 1},\]
which is measurable by Lemma \ref{lem:app1}. Then, we have $\nor{Y_i}_{\mathcal S_2}\leq 1 /\la$ almost surely, 
$\EE{Y_i}=(T+\la)^{-1}T$, $(1/n) \sum_{i=1}^n Y_i=(T+\la)^{-1}T_n$ and 
\begin{align*} 
\EE{\nor{Y_i}^2_{\mathcal S_2}} & = \int_\Omega \tr{Y_i(\omega)^\ast Y_i(\omega)} d\PP{\omega} = \int_X
    \tr{(T+\la)^{-2}(K_x\otimes K_x)} d\rho(x) \\
&=\tr{(T+\la)^{-2}T}\leq \nor{(T+\la)^{-1}}_{\infty}\tr{(T+\la)^{-1}T} \leq \frac{{\cal N}(\la)}{\la},
\end{align*}
where we have bounded the operator norm $\nor{(T+\la)^{-1}}_{\infty}$ by $1/\la$.
The result follows applying Lemma \ref{conc_lemma}.
\end{proof}

\section*{Acknowledgement} 
A.~T.~acknowledges the financial support of the Italian
Ministry of Education, University and Research (FIRB project
RBFR10COAQ). L.~R.~acknowledges the financial support of the Italian Ministry of Education, University and Research (FIRB project RBFR12M3AC).



\end{document}